\documentclass{article} \usepackage{conference,times}

\usepackage{amsmath,amsfonts,bm}

\def\eqref#1{equation~\ref{#1}}

\def\1{\bm{1}}

\DeclareMathAlphabet{\mathsfit}{\encodingdefault}{\sfdefault}{m}{sl}
\SetMathAlphabet{\mathsfit}{bold}{\encodingdefault}{\sfdefault}{bx}{n}

\usepackage{hyperref}
\usepackage{url}
\usepackage{amsthm}
\usepackage{wrapfig}
\usepackage{graphicx}
\usepackage{subcaption}
\usepackage{booktabs}
\usepackage{adjustbox}

\usepackage{booktabs}
\usepackage{caption}
\usepackage{xcolor}

\usepackage{multirow}
\usepackage{threeparttable}

\definecolor{errorpink}{HTML}{D89DAD}

\newcommand{\dnserror}[1]{\textcolor{black}{(}\textcolor{errorpink}{\ensuremath{#1}}\textcolor{black}{)}}

\newsavebox{\pretraintablebox}
\newsavebox{\pretrainimagebox}

\newsavebox{\pretraincurvebox}

\newsavebox{\turbulencetablebox}
\newsavebox{\turbulenceimagebox}
\newsavebox{\turbulencescaledtablebox}
\newsavebox{\turbulencescaledimagebox}

\newlength{\turbulencegap}

\newlength{\turbulencetablewidth}
\newlength{\turbulencetableheight}

\newlength{\turbulenceimagewidth}
\newlength{\turbulenceimageheight}

\newlength{\turbulencebodyheight}
\newlength{\turbulenceleftwidth}
\newlength{\turbulencerightwidth}

\ExplSyntaxOn

\fp_new:N \l__turbulence_table_aspect_fp
\fp_new:N \l__turbulence_image_aspect_fp

\cs_new_protected:Npn \turbulence_compute_common_height:
{
    \fp_set:Nn \l__turbulence_table_aspect_fp
    {
        \dim_to_fp:n { \turbulencetablewidth }
        /
        \dim_to_fp:n { \turbulencetableheight }
    }

    \fp_set:Nn \l__turbulence_image_aspect_fp
    {
        \dim_to_fp:n { \turbulenceimagewidth }
        /
        \dim_to_fp:n { \turbulenceimageheight }
    }

    \dim_set:Nn \turbulencebodyheight
    {
        \fp_to_dim:n
        {
            (
                \dim_to_fp:n { \linewidth }
                -
                \dim_to_fp:n { \turbulencegap }
            )
            /
            (
                \l__turbulence_table_aspect_fp
                +
                \l__turbulence_image_aspect_fp
            )
        }
    }
}

\NewDocumentCommand{\computeturbulencecommonheight}{}
{
    \turbulence_compute_common_height:
}

\ExplSyntaxOff

\title{Not Every Term Adds New Structure: \\Sobolev Novelty for Symbolic Regression}

\author{\textbf{Boxiao Wang\textsuperscript{1}\quad
        Kai Li\textsuperscript{1,2,*}\quad
        Yuheng Jing\textsuperscript{1}\quad
        Tianyi Liu\textsuperscript{3}}\\
    \textbf{Chen Li\textsuperscript{3}\quad
        Yifan Zhang\textsuperscript{1,2,4,*}\quad
        Jian Cheng\textsuperscript{1,2,5}}\\[4pt]
    {\normalfont\small
        \textsuperscript{1}\,C2DL, Institute of Automation,
        Chinese Academy of Sciences}\\
    {\normalfont\small
        \textsuperscript{2}\,School of Artificial Intelligence,
        University of Chinese Academy of Sciences}\\
    {\normalfont\small
        \textsuperscript{3}\,State Key Laboratory of Aerodynamics}\\
    {\normalfont\small
        \textsuperscript{4}\,University of Chinese Academy of Sciences,
        Nanjing}\\
    {\normalfont\small
        \textsuperscript{5}\,AiRiA}\\[3pt]
    {\normalfont\small
        \textsuperscript{*}\ Corresponding authors}
}

\usepackage{amsmath,amssymb,amsthm}

\theoremstyle{plain}
\newtheorem{sntheorem}{Theorem}[section]
\newtheorem{snlemma}[sntheorem]{Lemma}
\newtheorem{snproposition}[sntheorem]{Proposition}
\newtheorem{sncorollary}[sntheorem]{Corollary}

\theoremstyle{definition}
\newtheorem{sndefinition}[sntheorem]{Definition}

\theoremstyle{remark}

\theoremstyle{plain}

\begin{document}

\maketitle

\pagestyle{fancy}
\fancyhead{}
\fancyfoot{}
\fancyfoot[C]{\thepage}
\thispagestyle{fancy}

\begin{abstract}
Symbolic regression (SR) aims to discover compact and meaningful mathematical equations from data, but searching the vast combinatorial space of symbolic structures remains challenging. Existing methods typically guide this process using expression-level objectives, such as fitting error, which assess a candidate equation as a whole but provide little information about whether an individual term contributes genuinely new structure or is largely redundant with the rest of the expression. We introduce \textbf{Sobolev Novelty}, a term-level measure of structural independence for symbolic equations. For each term, we construct an empirical Sobolev signature from its function values and exact derivatives over the observed inputs, and quantify how much of this behavior cannot be reconstructed by the remaining terms. 
We further derive a theory-calibrated threshold, yielding a principled and tuning-free criterion for identifying structurally novel terms. Using this threshold, 92.6\% of terms in benchmark ground-truth equations exhibit sufficient structural novelty, compared with only 38.3\% on average for expressions produced by 15 SR methods, revealing a substantial gap between scientific equations and current SR solutions.
As a lightweight plug-in, Sobolev Novelty can be incorporated into diverse SR paradigms to support term pruning, search guidance, LLM feedback, and data selection, yielding consistent performance gains and demonstrating broad applicability.
\end{abstract}

\section{Introduction}

Symbolic regression (SR) aims to discover compact and interpretable mathematical equations directly from observed data~\citep{makke2024interpretable}. Unlike conventional regression, where the functional form is specified in advance, SR must search jointly over symbolic structures and numerical parameters. This flexibility makes SR particularly appealing for scientific discovery, where the goal is often not only to predict accurately but also to recover an interpretable mathematical relationship~\citep{udrescu2020ai,brunton2016discovering,rudy2017data}. At the same time, it creates a fundamental search challenge: the space of candidate equations grows combinatorially with expression length, operator choices, variables, and their compositions. A broad range of approaches---including genetic programming, reinforcement learning~\citep{petersen2019deep}, tree search~\citep{sun2022symbolic}, pretrained generative models~\citep{biggio2021neural}, and more recently LLM-based methods~\citep{shojaee2025llm}---have been developed to navigate this space.

Despite their different search mechanisms, existing SR methods are typically guided by objectives defined at the level of complete candidate equations, most notably fitting error. Such objectives are effective at answering whether a candidate equation fits the observations well, but provide much less information about the structural role of its individual terms. 
Consider an equation candidate
$s(x)=\sum_{i=1}^{m} b_i\phi_i(x),$
where $\phi_i$ denotes a symbolic term and $b_i$ its numerical coefficient. The presence of $\phi_i$ in a fitted expression does not necessarily imply that $\phi_i$ represents an independently distinguishable component. Over the observed domain, its behavior may already be well reproduced by a combination of the remaining terms. In this case, multiple symbolic components participate in the fit while describing largely overlapping behavior.

For scientific discovery, this distinction is more than a matter of expression redundancy. Individual terms in a discovered equation are often interpreted as distinct components of the underlying relationship. If the observable behavior attributed to one term can already be reproduced by the remaining terms, however, the data provide limited support for interpreting that term as an independently supported structure. An accurate expression can therefore admit a symbolic decomposition that is less identifiable than its predictive performance might suggest. This motivates the central question of this work: \textit{How much behavior does an individual term contribute beyond what is already represented by the rest of the expression?}
This requires addressing two questions: how to represent the behavior of a symbolic term over the observed domain, and how to quantify the component of that behavior not captured by the rest of the expression.

We introduce \textbf{Sobolev Novelty}, a term-level measure of structural independence for symbolic equations. For each term, we construct an empirical Sobolev representation from its function values and exact first-order derivatives over the observed inputs, capturing both its observed behavior and local variation. We then measure how much of this representation lies outside the span of the remaining terms. Concretely, Sobolev Novelty is the normalized leave-one-term-out reconstruction residual in this empirical Sobolev space. A low-novelty term is largely representable by the other terms, whereas a high-novelty term contributes a behavioral direction that cannot be readily reproduced through coefficient adjustment. Importantly, the measure is entirely candidate-internal: it depends only on the symbolic expression and the observed inputs, requiring neither target derivatives nor access to the ground-truth equation.

We further derive a \textbf{theory-calibrated threshold} for Sobolev Novelty, providing a principled operating point without algorithm- or dataset-specific threshold tuning. Using this threshold, we observe a pronounced difference between benchmark scientific equations and expressions discovered by existing SR methods: 92.6\% of terms in ground-truth equations exhibit sufficient structural novelty, compared with only 38.3\% on average for expressions produced by 15 SR methods. This gap suggests that current SR systems can recover formulas with competitive predictive quality while producing symbolic decompositions whose constituent terms are substantially less distinguishable than those in the reference scientific equations. In other words, predictive accuracy alone does not guarantee that the internal decomposition of a discovered equation is equally well supported by the data. This finding motivates using Sobolev Novelty not only as a diagnostic of discovered expressions, but also as an actionable signal for improving the SR process.

Sobolev Novelty is a \textbf{lightweight plug-in} that can be integrated into diverse SR paradigms at different stages of their pipelines. 
For example, search-based methods can use it to retain structurally informative terms and prune or deprioritize redundant ones. LLM-based SR can use term-level novelty as explicit feedback about which parts of a proposed expression add new behavior and which repeat what is already present. Pretrained generative SR models can use the metric to filter synthetic training equations with substantial term-level redundancy. Across representative SR paradigms, incorporating Sobolev Novelty yields consistent performance improvements, supporting its use as a broadly applicable complement to conventional equation-level objectives.

Taken together, our results suggest that predictive accuracy alone does not fully characterize the quality of a symbolic equation for scientific discovery. Beyond fitting the observations, it is also desirable to understand whether the constituent terms contribute distinguishable structural information over the observed domain. Sobolev Novelty provides a practical way to quantify and encourage this property, offering a complementary criterion for steering SR toward expressions that are not only accurate and compact, but also more structurally identifiable and scientifically interpretable.

\section{Preliminaries}

In SR, the learning task starts with a dataset:
$
\mathcal{D}
=
\{(x_n,y_n)\}_{n=1}^{N},
x_n\in\mathbb{R}^{d},\;
y_n\in\mathbb{R},
$
where $x_n$ denotes a $d$-dimensional input vector and $y_n$ is the
corresponding scalar output. The goal is to discover an analytic
expression $s(\cdot)$ such that $s(x_n)$ accurately approximates the
targets $y_n$, while generalizing well to unseen inputs.

Most SR methods can be viewed as involving two coupled components: structure search and evaluation. Structure search explores the discrete space of expressions, determining the variables, operators, and compositions that form a candidate equation. Different SR paradigms realize this process through, for example, genetic programming, reinforcement learning, tree search~\citep{shojaee2023transformer,kamienny2023deep}, pretrained generative models~\citep{meidani2024snip}, and LLM-based generation. Candidate structures may also contain numerical constants, which are typically optimized against the observed data after the symbolic structure is proposed.

The resulting expression is then evaluated according to predictive quality, often together with other global criteria such as expression complexity~\citep{la2021contemporary,imai2025call,yu2025symbolic}. These equation-level objectives provide the signal used to rank, select, or refine candidate expressions~\citep{yu2025beyond}. While effective for assessing the overall quality of a candidate, they do not directly characterize the structural contribution of its individual terms. Next, we introduce a term-level perspective and develop Sobolev Novelty as a complementary measure of structural independence. Further discussion of related work is provided in Appendix~\ref{app:related_work}.

\section{Method}

\subsection{From Term Representability to Sobolev Novelty}

\paragraph{When does a term add new behavior?}
We define the structural contribution of a term relative to the current candidate: a term contributes new behavior to the extent that it cannot be reconstructed from the remaining terms. Consider a candidate equation
$s(x)=\sum_{i=1}^{m} b_i\phi_i(x)$.
If $\phi_i$ is removed, the coefficients of the remaining terms can be refitted. Because these terms enter the expression through linear coefficients, any component of $\phi_i$ that can be reproduced by their linear combination can be absorbed through coefficient adjustment. The residual after the best such reconstruction is therefore the behavior of $\phi_i$ that is not already represented elsewhere in the candidate. This turns term-level novelty into a leave-one-term-out reconstruction problem.

\paragraph{How do we quantify term behavior?}
To make the preceding reconstruction problem computable, we require an empirical representation of the behavior contributed by each term over the observed inputs. Function values alone may be insufficient to distinguish terms that agree at the sampled points but exhibit different local variation. We therefore characterize this behavior along two complementary dimensions: function values and gradients. The value vector \(\phi_i(X)\) records the outputs produced by the term at the observed inputs, while the input gradient \(\nabla_x\phi_i(X)\) characterizes how the term varies locally with respect to the inputs. Because candidate terms are explicit symbolic functions, both quantities can be evaluated directly at valid inputs, without estimating derivatives of the unknown target. This joint value-and-gradient characterization naturally induces an empirical first-order \textit{Sobolev geometry}, in which differences between functions are measured through both their values and first derivatives~\citep{czarnecki2017sobolev}. This provides a more discriminative notion of term representability, since a term is considered redundant only when the remaining terms can reproduce not only its observed values but also its local variation.

Specifically, we define the empirical Sobolev representation of $\phi_i$ as
\begin{equation}
\Psi_X(\phi_i)
=
\frac{1}{\sigma_s}
\left[
\sqrt{\frac{\lambda_0}{N}}\phi_i(X),
\sqrt{\frac{\lambda_1}{Nd}}
\operatorname{vec}\!\left(\nabla_x\phi_i(X)\right)
\right],
\label{eq:sobolev_signature}
\end{equation}
where $N=|X|$ is the number of observed inputs, $d$ is the input dimension, and $\sigma_s^2=N^{-1}\sum_{n=1}^{N}s(x_n)^2$. Specifically, $\phi_i(X)\in\mathbb{R}^{N}$ forms the value block, while $\nabla_x\phi_i(X)\in\mathbb{R}^{N\times d}$ forms the gradient block, whose entries are stacked by $\operatorname{vec}$ into a vector in $\mathbb{R}^{Nd}$. Hence, $\Psi_X(\phi_i)\in\mathbb{R}^{N(d+1)}$. The weights $\lambda_0$ and $\lambda_1$ control the relative contributions of the two blocks, and we use $\lambda_0=\lambda_1=1$ by default. All terms within a candidate share the same scale $\sigma_s$, placing their representations on a common candidate-level scale.

Equation~\ref{eq:sobolev_signature} induces the empirical Sobolev inner product $\langle f,g\rangle_{S,X}=\Psi_X(f)^\top\Psi_X(g)$, with the corresponding seminorm $\|f\|_{S,X}=\|\Psi_X(f)\|_2$. The resulting representation turns the structural relation among symbolic terms into a finite-dimensional geometry that simultaneously captures their observed values and local trends.

\paragraph{Definition of Sobolev Novelty.}
Throughout this work, we use the first-order representation by default; extending to higher-order Sobolev representations is straightforward by appending the corresponding derivative blocks. In this representation space, a coherent reconstruction must use the same linear combination of the remaining terms to jointly reproduce both the value and gradient blocks of $\phi_i$. The component that remains unmatched under this joint reconstruction defines its novelty. Specifically, let
\begin{equation}
\Psi_{-i}
=
\left[
\Psi_X(\phi_1),\ldots,
\Psi_X(\phi_{i-1}),
\Psi_X(\phi_{i+1}),\ldots,
\Psi_X(\phi_m)
\right]
\end{equation}
collect the representations of all terms except $\phi_i$. We define the \emph{Sobolev Novelty} of $\phi_i$ as its normalized leave-one-term-out reconstruction residual:
\begin{equation}
\nu_i
=
\min_{c\in\mathbb{R}^{m-1}}
\frac{
\left\|
\Psi_X(\phi_i)-\Psi_{-i}c
\right\|_2
}{
\left\|\Psi_X(\phi_i)\right\|_2
}
=
\frac{
\left\|
(I-P_{-i})\Psi_X(\phi_i)
\right\|_2
}{
\left\|\Psi_X(\phi_i)\right\|_2
},
\label{eq:sobolev_novelty}
\end{equation}
where $P_{-i}=\Psi_{-i}\Psi_{-i}^{+}$ denotes the orthogonal projection onto the span of the remaining term representations. In practice, Eq.~\ref{eq:sobolev_novelty} is computed using rank-revealing least squares without explicitly forming the projection matrix.

The score satisfies $0\leq\nu_i\leq1$. Values near zero indicate that the remaining terms can reconstruct most of the value-and-derivative behavior of $\phi_i$ using a common coefficient vector. Larger values indicate a stronger component that remains outside the span of the other terms. Sobolev Novelty thus measures how much distinct empirical behavior a term adds to the current candidate. All quantities in Eq.~\ref{eq:sobolev_novelty} are determined by the symbolic terms and observed inputs, making the score candidate-internal and requiring no access to the ground-truth equation.

\subsection{A Theory-Calibrated Threshold}

Sobolev Novelty provides a continuous measure of structural independence, which raises a practical question: \emph{how much novelty is sufficient for a term to constitute an independently identifiable direction?} To answer this question, we normalize each nonzero Sobolev signature as $u_i=\Psi_X(\phi_i)/\|\Psi_X(\phi_i)\|_2$, collect the normalized signatures in $U=[u_1,\ldots,u_m]$, and define their Gram matrix as $G=U^\top U$. Let $\mathcal{S}_{-i}=\operatorname{span}\{\phi_j:j\neq i\}$ denote the function space spanned by the remaining terms.
We first interpret Sobolev Novelty from two complementary perspectives: the projection perspective and the explanation stability perspective, and build on these interpretations to derive a threshold for guiding SR.

Specifically, for nonzero term signatures and a nonsingular normalized Gram matrix $G$, the two perspectives are connected by
\begin{equation}
\label{eq:sobolev_perspectives}
\nu_i^2
=
\underbrace{
    \|u_i\|_2^2-\|P_{-i}u_i\|_2^2
}_{\text{Projection Perspective}}
=
\underbrace{
    (G^{-1})_{ii}^{-1}
}_{\text{Explanation Stability Perspective}}.
\end{equation}
The detailed proof is provided in Appendix~\ref{app:theory}.

\paragraph{Projection Perspective.} \emph{How much new behavior does a term add beyond the remaining terms?}  In the first equality of Eq.~\ref{eq:sobolev_perspectives}, the normalized Sobolev signature $u_i$ represents the term's values and trends, while $P_{-i}u_i$ is the optimal linear combination of the remaining terms, representing the component they can explain. Because the projected component and the residual vector are orthogonal, subtracting the squared norm of the projected component from that of $u_i$ yields exactly $\nu_i^2$, the squared norm of the residual vector that the remaining terms cannot represent.  Higher novelty therefore corresponds to a smaller projection norm, meaning that less of the term's representation is captured by the remaining terms, leaving a larger residual; conversely, lower novelty corresponds to a smaller residual. When the projection is zero, $\nu_i=1$, indicating that the term's signature is orthogonal to the span of the remaining signatures.

\paragraph{Explanation Stability Perspective.}
\emph{Can a small change in overall behavior substantially change how its components are interpreted?}
An equation may accurately reproduce the observed behavior while having an unstable decomposition into individual contributions. When several terms have similar values and local trends, their contributions can change substantially through mutual compensation even while the equation's overall behavior remains nearly unchanged. This makes it difficult to determine how much of that behavior should be attributed to each term. This resembles separating similar sources from a mixed signal: the combined signal may be well determined while the individual contributions remain difficult to distinguish stably. The second equality asks whether we can still tell what role each term plays when the equation's overall behavior barely changes. When several terms behave similarly, even a slight change in the equation's output values and local trends can substantially change how much of that behavior is attributed to term $i$. Under the normalized representation above, we hold the functional forms of all terms fixed and examine how their refitted coefficients respond to small changes in the equation's function values and gradients. For term $i$, the maximum amplification factor is the largest possible ratio of the magnitude of its coefficient change to the size of the overall change. The square of this factor is exactly $(G^{-1})_{ii}$. Since $(G^{-1})_{ii}=1/\nu_i^2$, lower Sobolev Novelty corresponds to a larger maximum amplification factor, meaning that the term's coefficient can change substantially even when the equation's overall behavior changes only slightly.

Building on these perspectives, we can limit this instability by controlling Sobolev Novelty. Specifically, we require the square of the maximum amplification factor to be no more than one order of magnitude above the orthogonal baseline of $1$. This gives $1/\nu_i^2\le 10$, or equivalently $\nu_i\ge 1/\sqrt{10}$, yielding the threshold $\tau_{\mathrm{id}}=1/\sqrt{10}$.

This threshold admits both geometric and explanation-stability interpretations. For any term meeting the threshold, the residual vector that cannot be reconstructed by the remaining terms has a squared norm of at least $10\%$ of that of the term's full representation. At the same time, coefficient changes induced by changes in the equation's overall behavior are subject to the amplification bound above. The threshold therefore allows some behavioral overlap while flagging as low novelty those terms whose residuals are too small relative to their full representations and whose coefficient changes can be excessively amplified. Empirical statistics on scientific equations in Section~\ref{exp1} further validate the practical relevance of this criterion.

\section{Experiments}
In this section, we evaluate Sobolev Novelty through a sequence of questions. We first examine whether benchmark ground-truth equations and expressions produced by existing SR methods exhibit a systematic gap in term-level structural novelty (Section~\ref{exp1}). We then assess whether the same signal can guide symbolic search across evolutionary search, tree search, and LLM-guided tree search, using GP, MCTS, and IGSR as representative cases (Section~\ref{sec:search_guidance}). Next, we evaluate whether Sobolev Novelty can improve pretrained generative SR by filtering low-novelty synthetic formulas before training while leaving the model architecture and training pipeline unchanged (Section~\ref{sec:pretraining}). Finally, we enter a real scientific setting and evaluate Sobolev Novelty in turbulence modeling, a core problem in fluid mechanics that underpins reliable prediction of complex flows across science and engineering (Section~\ref{sec:turbulence}).

\paragraph{Datasets.}
We primarily evaluate on
SRBench~\citep{la2021contemporary} across our experiments.
Our evaluation includes 133 white-box scientific equation tasks with
analytic ground-truth expressions, comprising 119 Feynman tasks and
14 Strogatz tasks, as well as 122 black-box regression tasks for which
no reference analytic expression is provided.

\paragraph{Evaluation metrics.}
We measure predictive accuracy using the standard coefficient of determination on the test set,
$R^2 = 1 - \sum_n (y_n-\hat{y}_n)^2 / \sum_n (y_n-\bar{y})^2$.
Expression complexity $C$ is measured by the length of the parsed expression tree, defined as the total number of operator, variable, and constant nodes for each evaluated candidate expression. At the structural level, we report the mean term Sobolev Novelty and the proportion of terms satisfying $\nu_i>\tau_{\mathrm{id}}$.

\subsection{A Structural Novelty Gap in Existing SR Methods}
\label{exp1}

\paragraph{Methods and experimental setup.}
To examine whether expressions discovered by existing SR methods exhibit term-level structural independence comparable to that of benchmark ground-truth equations, we analyze the outputs of 15 representative methods on the 133 SRBench white-box tasks. The selected methods span the major search and learning paradigms. Among population-based approaches, AFP~\citep{schmidt2010age} uses age--fitness Pareto selection, AFP-FE~\citep{schmidt2008coevolution}
augments it with co-evolved fitness estimates, EPLEX adopts $\epsilon$-lexicase selection, and GP-GOMEA~\citep{virgolin2021improving} applies gene-pool optimal mixing. GPlearn, Operon~\citep{burlacu2020operon}, PySR~\citep{cranmer2023interpretable}, and SBP-GP~\citep{virgolin2019linear} perform genetic-programming search with different variation, selection, semantic, and implementation strategies, while FEAT~\citep{la2018learning} evolves nonlinear feature transformations that are combined through a linear model. BSR~\citep{jin2019bayesian} performs Bayesian search over symbolic structures, and AIFeynman2~\citep{udrescu2020ai_2} decomposes problems using physics-inspired heuristics such as separability and symmetry. DSR~\citep{petersen2019deep} and RSRM~\citep{xu2023rsrm} construct expressions using learned reinforcement-learning policies, whereas the pretrained neural methods NeurSR~\citep{biggio2021neural} and E2ESR~\citep{kamienny2022end} map numerical observations directly to symbolic expression sequences.

\begin{wrapfigure}{r}{0.4\columnwidth}
    \centering
    \vspace{-0.6em}
    \includegraphics[width=\linewidth]
    {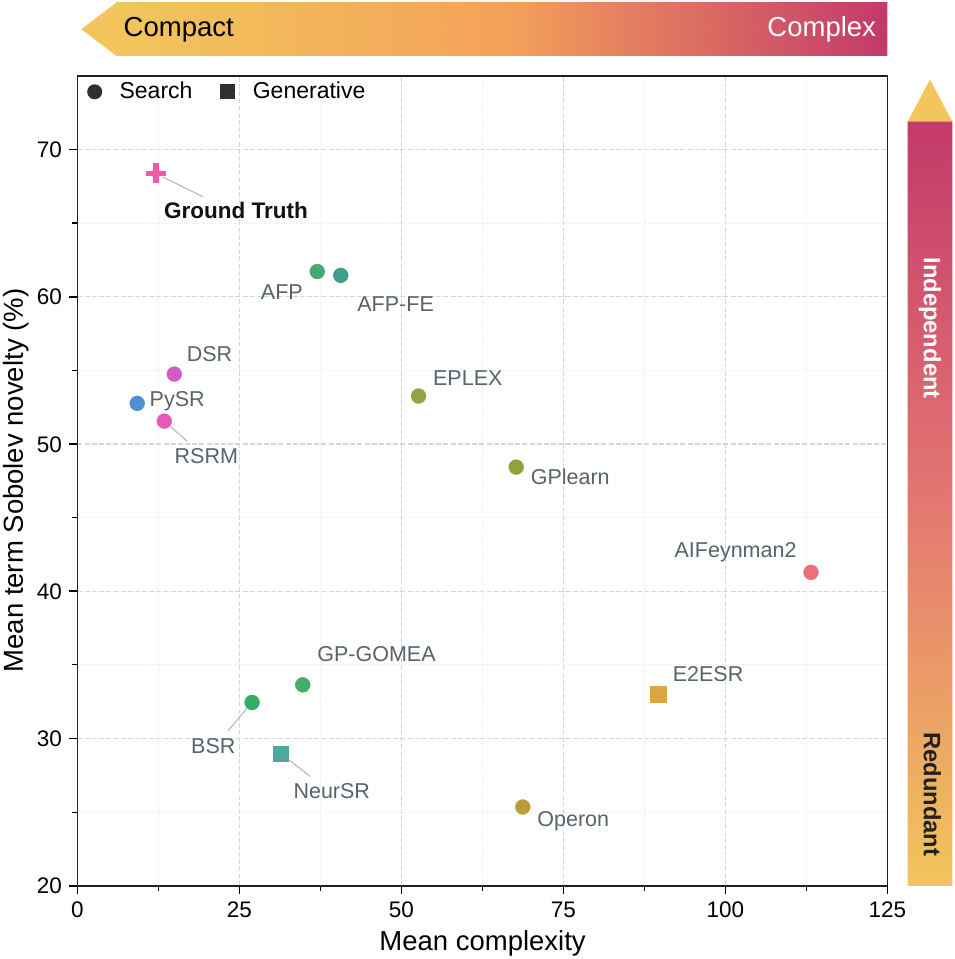}
    \caption{Mean complexity and term-level Sobolev Novelty across ground truth and SR methods.}
    \label{fig:complexity_novelty_gap}
    \vspace{-1.5em}
\end{wrapfigure}

Using this broad comparison set, we retain only discovered candidates with test $R^2 \geq 0.8$ and compute Sobolev Novelty for the ground-truth equation and all retained expressions using the same input samples within each task. The ground-truth expressions serve solely as post-hoc structural references and are not used to guide or supervise any of the methods. Complete per-method counts and term qualification rates are provided in
Appendix~\ref{app:qualification_statistics}.

\paragraph{Existing SR methods exhibit a substantial structural novelty gap.}
At the theory-calibrated threshold $\tau_{\mathrm{id}}$, $92.6\%$ of terms in the ground-truth equations satisfy the structural novelty criterion, compared with only $38.3\%$ on average across the 15 SR methods, yielding a gap of $54.3$ percentage points. The high qualification rate among ground-truth terms indicates that this criterion is consistent with the empirical structure of most benchmark scientific equations, whereas the much lower rate among discovered expressions reveals a substantial gap in term-level structural novelty.

\paragraph{High fitting quality and compact expressions do not guarantee high term-level novelty.}
As shown in Fig.~\ref{fig:complexity_novelty_gap}, among candidates with $R^2 \geq 0.8$, search-based methods still exhibit markedly different novelty--complexity profiles. AFP and AFP-FE~\citep{schmidt2008coevolution} are closest to the ground truth in mean term novelty, but the expressions they produce are substantially more complex than the ground-truth equations. DSR, RSRM, and PySR have mean expression complexity close to that of the ground-truth equations but exhibit intermediate novelty, whereas GP-GOMEA and Operon lie in lower-novelty regions. These contrasts show that satisfying the predictive-quality criterion, even with expression complexity comparable to that of the ground-truth equations, does not guarantee comparable term-level novelty. This motivates supplementing the original expression-level objectives with term-level feedback that directly evaluates whether a candidate term expands the current behavioral span.

\paragraph{Pretraining on random formulas can favor low-novelty expressions.}
NeurSR and E2ESR both occupy low-novelty regions, with E2ESR also showing substantially higher mean expression complexity than the ground-truth equations. These methods are pretrained on randomly generated formulas, yet random generation does not ensure structural independence among terms. The pretraining corpus can therefore contain low-novelty structures, which models may learn to reproduce even when generating predictively accurate expressions. This motivates using Sobolev Novelty to filter formulas containing low-novelty terms before pretraining, shifting the training distribution toward formulas with more structurally independent terms.

\subsection{Sobolev Novelty as Search Guidance}
\label{sec:search_guidance}

\paragraph{Experimental setups.}
To examine whether Sobolev Novelty can provide reusable term-level guidance across different symbolic search mechanisms, we integrate it into GP, MCTS, and IGSR~\citep{saveliev2026influence} while retaining the original expression-level objective of each method. GP generates candidate expressions through crossover and mutation; we use Sobolev Novelty to screen these candidates and favor those containing more structurally independent terms. MCTS grows a search tree through node selection and expansion; during expansion, Sobolev Novelty reranks newly generated child nodes so that branches introducing new term behavior are explored first. IGSR combines LLM-generated proposals with MCTS; before node expansion, we translate the novelty profile of the current terms into natural-language feedback that guides the LLM away from structural directions already represented by the current expression. We evaluate GP and MCTS on the 133 white-box and 122 black-box SRBench tasks, whereas IGSR is evaluated only on the black-box tasks to mitigate potential contamination from publicly available scientific formulas in the LLM's pretraining data. Full integration and implementation details are provided in Appendix~\ref{app:implementation}.

\paragraph{Sobolev Novelty improves accuracy and compactness on white-box tasks.}
As shown in Figure~\ref{fig:search_guidance_whitebox},
both Sobolev-guided variants move toward the lower-left region of the
accuracy--complexity plane, with a more pronounced shift for GP.
We attribute this joint improvement to finer-grained structural feedback.
In both backbones, the original predictive criterion first retains proposals
that are competitive in fitting the data, after which Sobolev Novelty
deprioritizes structures already covered by the current expression and favors
candidates with complementary value-and-derivative behavior.
This reduces the search and expression budgets spent on low-novelty terms and
focuses the search on directions that expand the current behavioral span.
Consequently, Sobolev Novelty directly targets term-level structural
independence, an objective distinct from the standard accuracy--complexity
trade-off, yet it synergistically improves both metrics.
Table~\ref{tab:app_whitebox_results} reports the complete numerical results.
Both SN-GP and SN-MCTS achieve higher mean term Sobolev Novelty than their
respective backbones, together with the accuracy and complexity results
summarized in Figure~\ref{fig:search_guidance_whitebox}.

\begin{figure*}[t]
    \centering
    \begin{subfigure}[t]{0.495\textwidth}
        \centering
        \includegraphics[width=\linewidth]{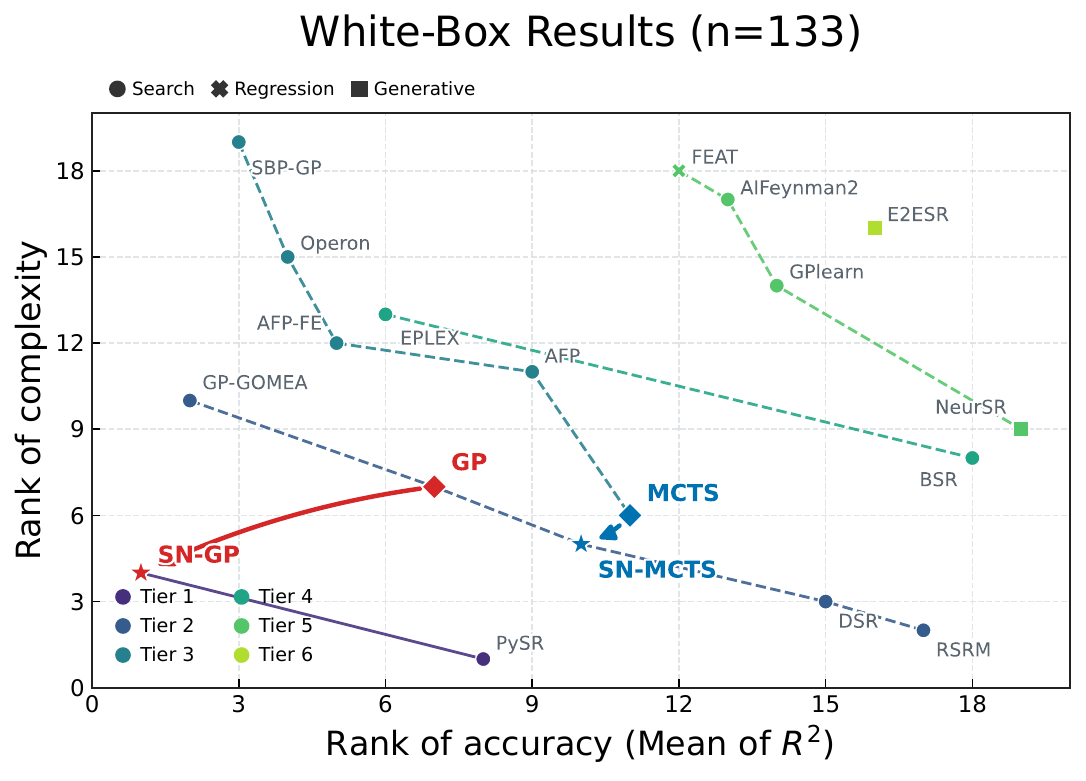}
        \caption{White-box results.}
        \label{fig:search_guidance_whitebox}
    \end{subfigure}
    \hfill
    \begin{subfigure}[t]{0.495\textwidth}
        \centering
        \includegraphics[width=\linewidth]{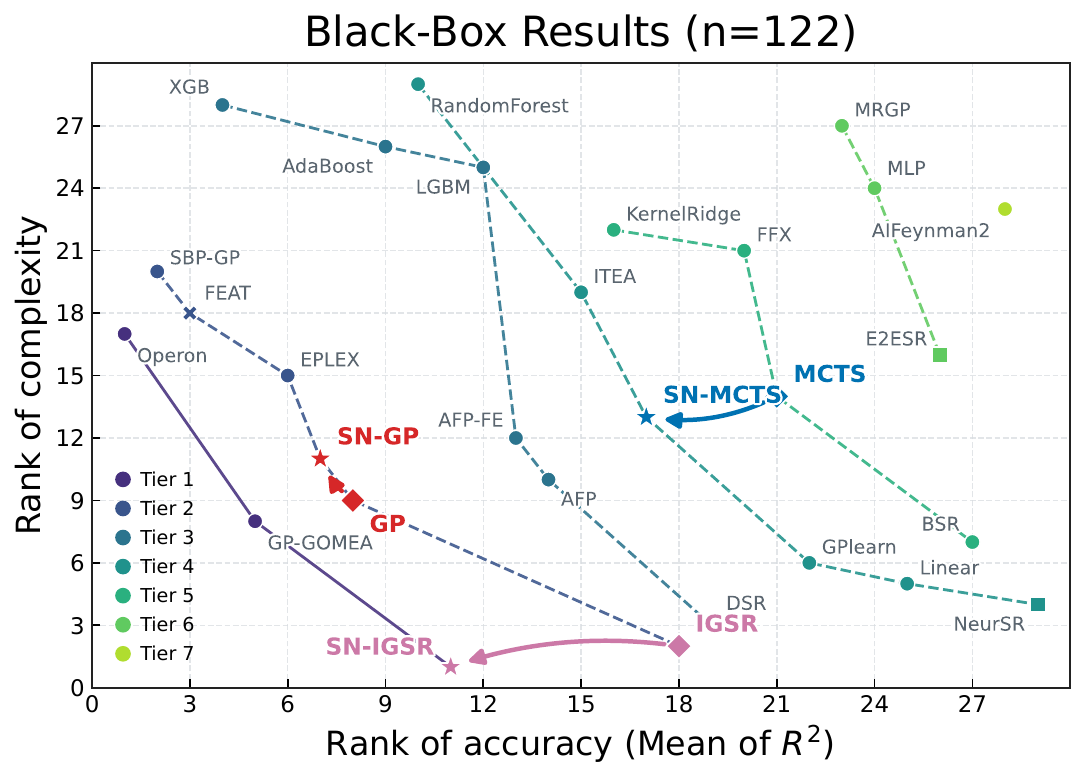}
        \caption{Black-box results.}
        \label{fig:search_guidance_blackbox}
    \end{subfigure}
    \caption{
        Accuracy--complexity rank comparisons on the SRBench white-box and
        black-box benchmarks.
        Lower is better on both axes.
        Dashed lines indicate Pareto tiers, and arrows connect each backbone
        with its Sobolev-guided variant.
        Accuracy ranks are computed from mean $R^2$, and complexity ranks from
        mean expression complexity.
    }
    \label{fig:search_guidance}
\end{figure*}

\paragraph{Sobolev guidance remains effective on black-box tasks, with backbone-dependent complexity effects.}
Figure~\ref{fig:search_guidance_blackbox}
summarizes the results on 122 black-box tasks.
SN-GP, SN-MCTS, and SN-IGSR all achieve higher mean $R^2$ than their
respective backbones, showing that term-level structural feedback improves
search performance across these three methods.
SN-MCTS moves toward the lower-left region of the accuracy--complexity plane,
improving both accuracy and compactness, whereas SN-GP gains accuracy with a
modest increase in complexity.

\paragraph{Sobolev feedback improves LLM-guided search without enlarging expressions.}
IGSR already employs a drop-one influence score to assess whether a generated
term is predictively useful for fitting the current target.
We preserve this predictive selection mechanism and use Sobolev Novelty only
to guide which structural directions the LLM should propose next.
The two signals therefore play distinct roles: predictive influence determines
which existing terms should be retained, whereas Sobolev Novelty identifies
which behaviors are already represented and which directions remain to be
explored.
Under the same task set, random seed, and search budget, SN-IGSR moves
substantially leftward along the accuracy axis in
Figure~\ref{fig:search_guidance}\subref{fig:search_guidance_blackbox},
attaining a better accuracy rank while preserving the same complexity.
This shows that the improvement is not obtained by generating more terms.
More broadly, this result demonstrates that Sobolev Novelty need not be tied
to a numerical reward: the same term-level structural signal can be translated
into natural-language feedback and directly used to improve LLM-guided
symbolic search.

\subsection{Sobolev Novelty as Pretraining Guidance}
\label{sec:pretraining}

\begin{wrapfigure}{r}{0.38\textwidth}
    \centering
    \includegraphics[
        width=\linewidth
    ]{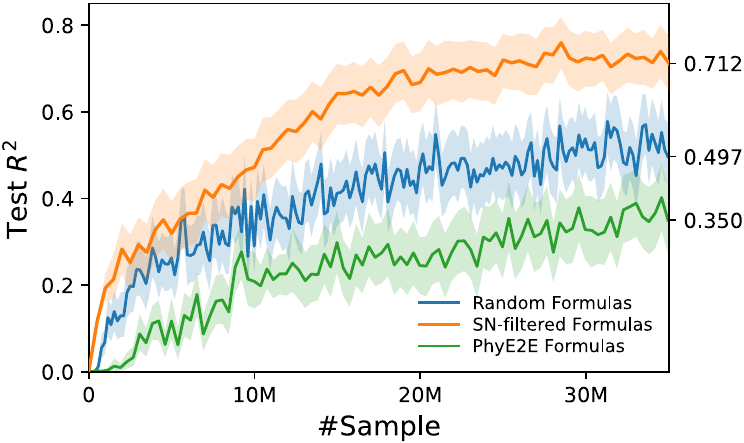}
    \caption{
        Generalization performance of E2ESR under different
        pretraining formula distributions. Endpoint $R^2$
        values are annotated on the right.
    }
    \label{fig:pretraining_curve}
    \vspace{-1em}
\end{wrapfigure}

Pretrained generative SR methods commonly rely on large-scale randomly generated formula--data pairs~\citep{biggio2021neural,kamienny2022end,meidani2024snip}. Although random formula generators provide a continuous supply of diverse structures, they do not ensure that the additive terms within each formula contribute distinct functional behavior. Consequently, a random pretraining corpus may contain many terms that are syntactically different yet strongly overlap with one another in their function values and local variation patterns. We address this problem through Sobolev Novelty-guided rejection sampling. Before a generated formula enters pretraining, we filter out samples containing low-novelty terms, allowing the model to learn from formulas whose constituent terms provide more distinguishable structural behavior.

\paragraph{Experimental setups.}
We use Sobolev Novelty for rejection sampling at the entrance to the E2ESR pretraining pipeline. Each formula is decomposed into additive terms; single-term formulas are accepted directly, whereas multi-term formulas are retained only when $\min_i\nu_i>\tau_{\mathrm{id}}$. The tokenizer, training objective, optimizer, and inference pipeline remain unchanged throughout the pretraining process. Under the same protocol, we compare Random, SN-filtered, and PhyE2E formulas. The PhyE2E formulas are taken from the pretraining corpus introduced by Ying et al.~\citeyearpar{ying2025neural}. We train three independent runs for each setting and evaluate checkpoints from 1M to 35M accepted samples on the 133 SRBench white-box tasks using mean $R^2$.

\paragraph{Sobolev filtering improves pretraining generalization.}
Figure~\ref{fig:pretraining_curve} shows the generalization performance of E2ESR under different pretraining formula distributions. Compared with Random and the PhyE2E-based baseline, models trained on SN-filtered formulas attain a substantially higher $R^2$ during the later stages of training and maintain a clear final-performance advantage under the same pretraining budget.
At the final 35M checkpoint, SN-filtered, Random, and PhyE2E formulas achieve mean $R^2$ values of $0.712$, $0.497$, and $0.350$, respectively. Relative to the $0.497$ obtained with Random formulas under the same budget, this corresponds to an absolute improvement of $0.215$ and a relative gain of $43.3\%$. Because this experiment changes only the distribution of formulas entering pretraining, without modifying the E2ESR architecture, training objective, or inference procedure, the results show that term-level pretraining guidance alone can substantially improve the final generalization performance of generative SR methods.

\subsection{Sobolev Novelty Enables Scientific Discovery in Turbulence Modeling}
\label{sec:turbulence}

To assess Sobolev Novelty in a real scientific discovery setting, we apply it to turbulence modeling for periodic-hill flow. The task extends beyond isolated input--output fitting because the discovered constitutive equation must close the governing equations~\citep{schmelzer2020discovery}. Periodic-hill flow presents a complex geometry in which adverse pressure gradients and streamline curvature generate separated shear layers, recirculation bubbles, and downstream reattachment, accompanied by pronounced non-equilibrium effects and Reynolds-stress anisotropy \citep{xiao2020flows}. The final physical prediction is obtained only after the discovered closure is embedded in a RANS solver. In this coupled setting, closure errors can propagate and be amplified; a closure with high fitting accuracy can still yield an ill-conditioned coupled system or severely distorted flow-field predictions \citep{wu2019reynolds}. The evaluation tests whether a discovered equation can still produce accurate system-level physical predictions after it is integrated into the governing equations and exposed to the input distribution induced by solver feedback.

\paragraph{Experimental setups.}
We consider the basic MCTS framework introduced above and DSRRANS~\citep{tang2023discovering}, a domain-specific SR method for turbulence modeling. Conventional RANS serves as the fluid-mechanics baseline, and DNS provides the high-fidelity reference. Appendix~\ref{app:turbulence_details} provides details of data processing, constitutive modeling, and solver configurations.

\paragraph{Sobolev Novelty substantially strengthens scientific discovery.}
We report two key physical quantities for the periodic-hill flow: the reattachment position and the maximum recirculation-zone velocity magnitude. These quantities characterize the downstream extent of the separation region and its recirculation strength, respectively. Accuracy is measured by the absolute deviation from DNS. As shown in Table~\ref{tab:turbulence_topology}, SN-DSRRANS achieves the closest agreement with DNS for both quantities. Relative to DSRRANS, it reduces the absolute reattachment-position error from 1.91 to 0.53 and the recirculation-velocity error from $2.52\times10^{-3}$ to $0.38\times10^{-3}\,\mathrm{m/s}$, corresponding to reductions of 72.3\% and 84.9\%, respectively. The MCTS framework exhibits the same pattern: SN-MCTS reduces the corresponding errors by 7.1\% and 60.8\%. Moreover, SN-MCTS is closer to DNS than both RANS and DSRRANS for both quantities. This is especially notable because MCTS contains no turbulence-specific inductive bias. With term-level structural novelty as the only additional guidance, this basic symbolic search framework discovers a constitutive relation whose a posteriori predictions surpass those of both a traditional turbulence model and a domain-specific SR method on the reported physical observables. These results provide direct evidence that Sobolev Novelty can endow even a basic symbolic search framework with strong scientific discovery capability.

\begin{figure}[t]
    \centering

    \captionsetup{
        font=footnotesize,
        labelfont=bf,
        justification=justified,
        singlelinecheck=false,
        skip=0pt
    }
    \captionsetup[table]{position=bottom}
    \captionsetup[figure]{position=bottom}

\begin{lrbox}{\turbulencetablebox}
    \scriptsize
    \setlength{\tabcolsep}{3.0pt}
    \renewcommand{\arraystretch}{1.20}

    \begin{tabular}{@{}lcc@{}}
        \toprule

        Method
        & \shortstack{
            Reattachment Point\\
            ($x/H$)
        }
        & \shortstack{
            Max Recirculation Vel.\\
            ($U_x$, $10^{-3}\,\mathrm{m/s}$)
        }
        \\

        \midrule

        DNS
        & 5.28
        & 7.39
        \\

        RANS
        & 7.36\,\dnserror{\ensuremath{\Delta\,2.08}}
        & 5.81\,\dnserror{\ensuremath{\Delta\,1.58}}
        \\

        DSRRANS
        & 7.19\,\dnserror{\ensuremath{\Delta\,1.91}}
        & 9.91\,\dnserror{\ensuremath{\Delta\,2.52}}
        \\

        SN-DSRRANS
        & \textbf{4.75}\,
          \dnserror{\ensuremath{\boldsymbol{\Delta\,0.53}}}
        & \textbf{7.01}\,
          \dnserror{\ensuremath{\boldsymbol{\Delta\,0.38}}}
        \\

        MCTS
        & 4.43\,\dnserror{\ensuremath{\Delta\,0.85}}
        & 9.20\,\dnserror{\ensuremath{\Delta\,1.81}}
        \\

        SN-MCTS
        & \underline{4.49}\,
          \dnserror{\ensuremath{\underline{\Delta\,0.79}}}
        & \underline{8.10}\,
          \dnserror{\ensuremath{\underline{\Delta\,0.71}}}
        \\

        \bottomrule
    \end{tabular}
\end{lrbox}

    \sbox{\turbulenceimagebox}{\includegraphics{
            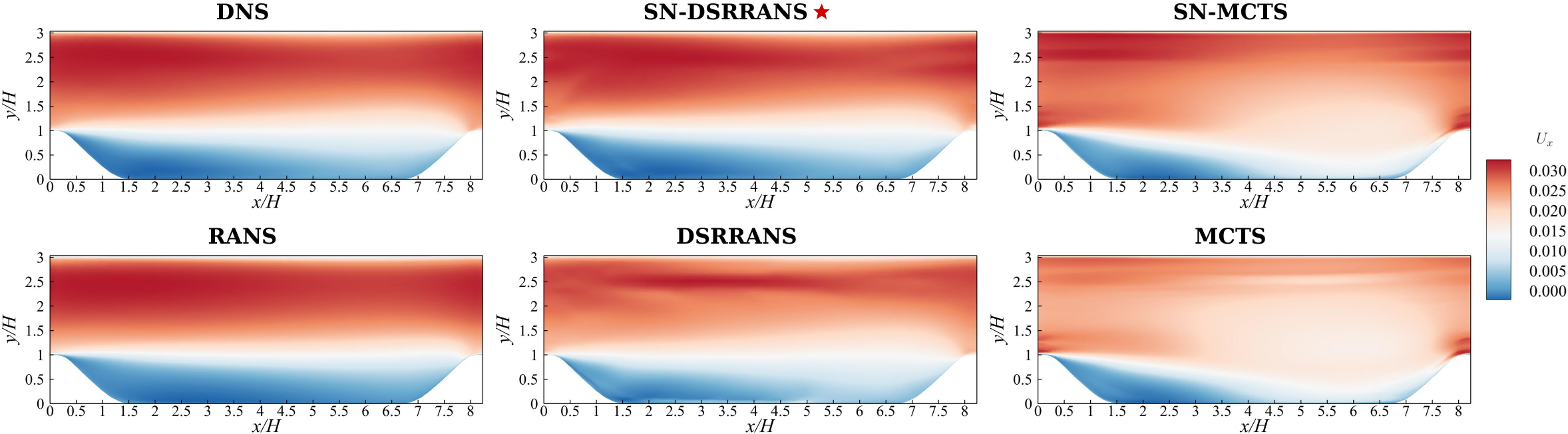
        }}

    \setlength{\turbulencetablewidth}{\wd\turbulencetablebox
    }

    \setlength{\turbulencetableheight}{\dimexpr
        \ht\turbulencetablebox
        +
        \dp\turbulencetablebox
        \relax
    }

    \setlength{\turbulenceimagewidth}{\wd\turbulenceimagebox
    }

    \setlength{\turbulenceimageheight}{\dimexpr
        \ht\turbulenceimagebox
        +
        \dp\turbulenceimagebox
        \relax
    }

    \setlength{\turbulencegap}{0.02\linewidth}

    \computeturbulencecommonheight

    \sbox{\turbulencescaledtablebox}{\resizebox*{!}{\turbulencebodyheight}{\usebox{\turbulencetablebox}}}

    \sbox{\turbulencescaledimagebox}{\resizebox*{!}{\turbulencebodyheight}{\usebox{\turbulenceimagebox}}}

    \setlength{\turbulenceleftwidth}{\wd\turbulencescaledtablebox
    }

    \setlength{\turbulencerightwidth}{\wd\turbulencescaledimagebox
    }

    \begin{minipage}[t]{\turbulenceleftwidth}
        \vspace{0pt}
        \centering

        \usebox{\turbulencescaledtablebox}

        \par\nointerlineskip
        \vspace{3pt}

        \captionof{table}{
    Comparison of reattachment position and maximum
    recirculation-zone velocity in periodic-hill flow.
}
        \label{tab:turbulence_topology}
    \end{minipage}\hfill
    \begin{minipage}[t]{\turbulencerightwidth}
        \vspace{0pt}
        \centering

        \usebox{\turbulencescaledimagebox}

        \par\nointerlineskip
        \vspace{3pt}

        \captionof{figure}{
            A posteriori streamwise velocity fields for periodic-hill flow
            obtained with DNS, RANS, DSRRANS, SN-DSRRANS, MCTS, and SN-MCTS.
        }
        \label{fig:turbulence_velocity}
    \end{minipage}
    \vspace{-1em}
\end{figure}

\paragraph{Sobolev Novelty improves a posteriori velocity-field predictions.}
The streamwise velocity fields in Figure~\ref{fig:turbulence_velocity} corroborate the quantitative results. The DSRRANS prediction contains an abrupt and locally concentrated high-velocity band in the upper channel that is absent from DNS. SN-DSRRANS suppresses this anomalous concentration and recovers a broader high-velocity region, a more continuous streamwise distribution, and a smoother transition toward the channel interior, while preserving the morphology and vortical structures of the periodic-hill separation region. These fields are obtained after the constitutive relation is embedded into the RANS solver and coupled with the governing equations. Their agreement with DNS consequently evaluates system-level physical fidelity under coupled simulation, beyond pointwise fitting of the training data. Together, the quantitative and field-level results show that Sobolev Novelty guides SR toward constitutive relations that close the governing equations and reproduce key turbulent-flow structures with high fidelity, providing direct evidence of scientific discovery capability in a high-fidelity physical modeling task.

\section{Conclusion}

We introduce Sobolev Novelty, a term-level measure of structural independence that jointly characterizes function values and derivatives to quantify the distinct behavior each term contributes beyond the remaining expression. We prove its connections to independent Sobolev energy, deletion-and-refit cost, and coefficient identifiability, establishing a theory-calibrated criterion for structural novelty. Empirical analysis reveals a substantial novelty gap between scientific equations and current SR solutions, validating the practical relevance of this perspective. Across a series of experiments, Sobolev Novelty consistently improves search-based and generative SR methods through search guidance and pretraining data selection, demonstrating broad applicability. Going further, we enter a real scientific setting through turbulence modeling in fluid mechanics. The discovered constitutive relations yield more accurate flow predictions in coupled physical simulations, demonstrating Sobolev Novelty's ability to advance scientific discovery. Future work will explore higher-order Sobolev representations and extend this framework to a broader range of scientific discovery tasks.

\bibliography{conference}
\bibliographystyle{conference}

\appendix

\section{Theoretical Foundations and Complete Proofs}
\label{app:theory}
\numberwithin{equation}{section}
\setcounter{equation}{0}

This appendix provides the complete derivation of the Projection
Perspective and the Explanation Stability Perspective, and establishes
the resulting theory-calibrated threshold. We then prove the exact
deletion-and-refit characterization stated in the main text.

\subsection{Problem Setting and Empirical Sobolev Geometry}
\label{app:geometry}

Let $X=\{x_n\}_{n=1}^{N}$ be the observed inputs, where $N,d\ge1$
are integers, $x_n\in\Omega\subset\mathbb R^d$, and $\Omega$ is open.
Fix a candidate expression and its additive decomposition
\begin{equation}
    s(x)=\sum_{j=1}^{m}b_j\phi_j(x),
    \qquad b_j\in\mathbb R,\quad
    \phi_j\in C^1(\Omega),\quad m\ge1.
    \label{eq:app_decomposition}
\end{equation}
The terms, including their internal parameters, are fixed throughout
the analysis. Only their outer linear coefficients are varied in
reconstruction and refitting. All signatures use the same input samples,
input coordinates, and positive weights $\lambda_0,\lambda_1$.
As in the main text, define
\begin{equation}
    \sigma_s=
    \left(\frac1N\sum_{n=1}^{N}s(x_n)^2\right)^{1/2}>0.
    \label{eq:app_scale}
\end{equation}
This scale is fixed by the original candidate $s$ and is used for every
function compared below; it is not recomputed when coefficients change.

\begin{sndefinition}[Empirical Sobolev signature]
\label{def:app_signature}
For $f\in C^1(\Omega)$, define the column vector
\begin{equation}
    \Psi_X(f)=\frac1{\sigma_s}
    \begin{bmatrix}
        \sqrt{\lambda_0/N}\,f(X)\\[2pt]
        \sqrt{\lambda_1/(Nd)}\,
        \operatorname{vec}(\nabla_x f(X))
    \end{bmatrix}
    \in\mathbb R^p,
    \qquad p=N(d+1),
    \label{eq:app_signature}
\end{equation}
where $f(X)=(f(x_1),\ldots,f(x_N))^\top$ and
$\nabla_x f(X)\in\mathbb R^{N\times d}$.
The vectorization order is fixed across all functions.
Define
\begin{equation}
    \langle f,g\rangle_{S,X}=\Psi_X(f)^\top\Psi_X(g),
    \qquad
    \|f\|_{S,X}=\|\Psi_X(f)\|_2.
    \label{eq:app_seminorm}
\end{equation}
\end{sndefinition}

\begin{snlemma}[Linearity and the induced seminorm]
\label{lem:app_linearity}
The map $\Psi_X$ is linear. The form
$\langle\cdot,\cdot\rangle_{S,X}$ is symmetric, bilinear, and positive
semidefinite, and $\|\cdot\|_{S,X}$ is a seminorm.
\end{snlemma}
\begin{proof}
Function evaluation and differentiation are linear. Since the weights
and scale are fixed, for any $a,b\in\mathbb R$ and
$f,g\in C^1(\Omega)$,
\[
    \Psi_X(af+bg)=a\Psi_X(f)+b\Psi_X(g).
\]
The assertions about the bilinear form follow from the Euclidean
inner product. Absolute homogeneity and the triangle inequality follow
from the Euclidean norm. Finite samples of values and derivatives need
not distinguish distinct functions, so positive definiteness on
$C^1(\Omega)$ is not guaranteed.
\end{proof}

All orthogonal projections below are taken in the finite-dimensional
signature space with its Euclidean inner product. This remains
well-defined even though the induced function-space quantity is only
a seminorm. We assume $\Psi_X(\phi_j)\ne0$ for each term to be normalized.
Neither target values nor target derivatives are needed to construct
these signatures or to compute novelty.

\subsection{Sobolev Novelty and the Projection Perspective}
\label{app:projection}

Write
\begin{equation}
    \begin{gathered}
        \psi_j=\Psi_X(\phi_j),\qquad
        \ell_j=\|\psi_j\|_2>0,\qquad u_j=\psi_j/\ell_j,\\
        U=[u_1,\ldots,u_m],\qquad G=U^\top U.
    \end{gathered}
    \label{eq:app_normalized}
\end{equation}
For a fixed index $i$, let $\Psi_{-i}$ and $U_{-i}$ be the matrices
obtained by removing the $i$th columns from
$[\psi_1,\ldots,\psi_m]$ and $U$, respectively. Define
\begin{equation}
    W_{-i}=\operatorname{span}\{u_j:j\ne i\}
          =\operatorname{range}(\Psi_{-i}),
    \qquad P_{-i}=\Psi_{-i}\Psi_{-i}^{+},
    \label{eq:app_projection}
\end{equation}
where $+$ denotes the Moore--Penrose pseudoinverse.
The Moore--Penrose identities give $P_{-i}^\top=P_{-i}$,
$P_{-i}^2=P_{-i}$, and $\operatorname{range}(P_{-i})=W_{-i}$;
thus $P_{-i}$ is the orthogonal projection onto $W_{-i}$.
For $m=1$, the reference space is $\{0\}$ and $P_{-i}=0$.

\begin{sndefinition}[Sobolev Novelty]
\label{def:app_novelty}
The Sobolev Novelty of term $\phi_i$ relative to the remaining terms is
\begin{equation}
    \nu_i=\min_{c\in\mathbb R^{m-1}}
    \frac{\|\psi_i-\Psi_{-i}c\|_2}{\ell_i}.
    \label{eq:app_novelty}
\end{equation}
The same coefficient vector $c$ acts on both the value and gradient
blocks. Let $r_i=(I-P_{-i})u_i$ denote the residual vector of the
normalized signature.
\end{sndefinition}

\begin{snproposition}[Projection residual and independent energy]
\label{prop:app_projection}
Without requiring $U$ to have full column rank,
\begin{equation}
    \nu_i=\|r_i\|_2,
    \qquad
    \nu_i^2=\|u_i\|_2^2-\|P_{-i}u_i\|_2^2
           =1-\|P_{-i}u_i\|_2^2.
    \label{eq:app_energy}
\end{equation}
Consequently, $0\le\nu_i\le1$; $\nu_i=0$ if and only if
$u_i\in W_{-i}$, and $\nu_i=1$ if and only if $u_i\perp W_{-i}$.
In particular, a nonzero single-term signature has novelty $1$.
\end{snproposition}
\begin{proof}
For any $c$, the two vectors in the decomposition
\[
    \psi_i-\Psi_{-i}c
    =(I-P_{-i})\psi_i+(P_{-i}\psi_i-\Psi_{-i}c)
\]
are orthogonal. Hence
\[
    \|\psi_i-\Psi_{-i}c\|_2^2
    =\|(I-P_{-i})\psi_i\|_2^2
     +\|P_{-i}\psi_i-\Psi_{-i}c\|_2^2.
\]
Because $P_{-i}\psi_i\in\operatorname{range}(\Psi_{-i})$, the second
term can be made zero. Dividing the minimum residual norm by $\ell_i$
gives $\nu_i=\|(I-P_{-i})u_i\|_2$.
The orthogonal decomposition $u_i=P_{-i}u_i+r_i$ and
$\|u_i\|_2=1$ imply~\eqref{eq:app_energy}.
The range and endpoint characterizations follow immediately.
\end{proof}

Here, energy means the squared Euclidean norm of a Sobolev signature.
Thus $\nu_i^2$ is exactly the fraction of the term's empirical
value-and-gradient energy that cannot be reconstructed by the remaining
terms. Increasing novelty decreases the projected energy and increases
the residual energy.

\subsection{Term Identifiability and the Explanation Stability Perspective}
\label{app:proof_identifiability}

In this subsection and the threshold derivation below, assume that $U$
has full column rank, equivalently that $G$ is nonsingular, as in the
main text. To formalize attribution of overall behavior to individual
terms, keep $U$ fixed and consider, for $v\in\mathbb R^p$,
\begin{equation}
    \widehat\alpha(v)=\operatorname*{arg\,min}_{\alpha\in\mathbb R^m}
    \|v-U\alpha\|_2^2.
    \label{eq:app_attribution_ls}
\end{equation}
Since $G$ is positive definite, this least-squares problem has a unique
minimizer. Define the contribution attributed to term $i$ as
\begin{equation}
    q_i(v)=\widehat\alpha_i(v)u_i.
    \label{eq:app_attribution}
\end{equation}
For the original candidate,
$\Psi_X(s)=\sum_j b_j\ell_j u_j$, so uniqueness gives
$q_i(\Psi_X(s))=b_i\Psi_X(\phi_i)$.
Thus $q_i$ describes the contribution of term $i$ to the candidate's
value-and-gradient behavior.

\begin{sntheorem}[Exact attribution sensitivity]
\label{thm:app_identifiability}
Under the full-column-rank assumption above, $\nu_i>0$ for every term $i$,
and
\begin{equation}
    \widehat\alpha_i(v)=\frac{r_i^\top v}{\nu_i^2}.
    \label{eq:app_attribution_formula}
\end{equation}
The worst-case attribution sensitivity satisfies
\begin{equation}
    \kappa_i:=\sup_{h\in\mathbb R^p\setminus\{0\}}
    \frac{\|q_i(v+h)-q_i(v)\|_2}{\|h\|_2}
    =\frac1{\nu_i}.
    \label{eq:app_attribution_sensitivity}
\end{equation}
The supremum is attained by $h=t r_i$ for any $t\ne0$.
\end{sntheorem}
\begin{proof}
Full column rank implies $u_i\notin W_{-i}$, so
Proposition~\ref{prop:app_projection} gives $\nu_i>0$.
Fix the $i$th coefficient at $a\in\mathbb R$ and minimize over the others.
Orthogonal projection gives
\begin{align}
    \min_{c\in\mathbb R^{m-1}}
    \|v-a u_i-U_{-i}c\|_2^2
    &=\|(I-P_{-i})(v-a u_i)\|_2^2 \notag\\
    &=\|(I-P_{-i})v-a r_i\|_2^2.
    \label{eq:app_partial_minimization}
\end{align}
Since $r_i\perp W_{-i}$, the final expression equals
\[
    \|(I-P_{-i})v\|_2^2-2a r_i^\top v+a^2\nu_i^2.
\]
This quadratic is strictly convex in $a$, and its unique minimizer is
$a=r_i^\top v/\nu_i^2$.
The inner minimization is attained for every $a$, so this is the $i$th
coordinate of the joint minimizer, proving
\eqref{eq:app_attribution_formula}.

Consequently,
\[
    q_i(v+h)-q_i(v)=\frac{r_i^\top h}{\nu_i^2}u_i.
\]
Using $\|u_i\|_2=1$ and the Cauchy--Schwarz inequality,
\[
    \|q_i(v+h)-q_i(v)\|_2
    =\frac{|r_i^\top h|}{\nu_i^2}
    \le\frac{\|r_i\|_2\|h\|_2}{\nu_i^2}
    =\frac{\|h\|_2}{\nu_i}.
\]
Equality holds for every nonzero scalar multiple of $r_i$.
Since $r_i=u_i-P_{-i}u_i\in\operatorname{range}(U)$, equality is also
attainable by changing only the coefficients of the fixed terms.
\end{proof}

\begin{sncorollary}[Normalized Gram-matrix characterization]
\label{cor:app_gram}
Under the same full-column-rank assumption, for every term $i$,
\begin{equation}
    \nu_i^2
    =1-\|P_{-i}u_i\|_2^2
    =\frac1{(G^{-1})_{ii}},
    \qquad
    \kappa_i^2=(G^{-1})_{ii}.
    \label{eq:app_gram_identity}
\end{equation}
\end{sncorollary}
\begin{proof}
For $m=1$, $G=[1]$ and $\nu_i=1$, so the result is immediate.
For $m\ge2$, reorder the columns so that term $i$ is last and write
\[
    G=\begin{bmatrix}A&g\\g^\top&1\end{bmatrix},
    \qquad A=U_{-i}^\top U_{-i},\qquad g=U_{-i}^\top u_i.
\]
Full column rank implies that $A$ is positive definite and
$u_i\notin W_{-i}$. Hence
$P_{-i}=U_{-i}A^{-1}U_{-i}^\top$, and
\[
    \nu_i^2=1-u_i^\top P_{-i}u_i=1-g^\top A^{-1}g>0.
\]
To obtain the corresponding diagonal entry of $G^{-1}$, solve
\[
    \begin{bmatrix}A&g\\g^\top&1\end{bmatrix}
    \begin{bmatrix}w\\t\end{bmatrix}
    =\begin{bmatrix}0\\1\end{bmatrix}.
\]
The first block equation yields $w=-A^{-1}gt$, and the second gives
$(1-g^\top A^{-1}g)t=1$.
Thus $(G^{-1})_{ii}=t=1/\nu_i^2$.
Combining this identity with Proposition~\ref{prop:app_projection}
and Theorem~\ref{thm:app_identifiability} proves the claim.
\end{proof}

Because $u_i$ has unit norm,
$\|q_i(v+h)-q_i(v)\|_2
=|\widehat\alpha_i(v+h)-\widehat\alpha_i(v)|$.
Thus the coefficient-sensitivity characterization is exactly an
attribution-sensitivity characterization in the normalized signature
representation.

\subsection{Derivation of the Theory-Calibrated Threshold}
\label{app:threshold_derivation}

Under the full-column-rank assumption of
Appendix~\ref{app:proof_identifiability},
Theorem~\ref{thm:app_identifiability} gives $\kappa_i^2=1/\nu_i^2$.
If $u_i$ is orthogonal to the remaining term span, then $\nu_i=1$ and
$\kappa_i=1$. Hence the squared attribution-sensitivity inflation
relative to this orthogonal case is
\begin{equation}
    \mathcal A_i=\kappa_i^2=\frac1{\nu_i^2}.
    \label{eq:app_inflation}
\end{equation}
We adopt the criterion stated in the main text: the squared attribution
sensitivity of a retained term should remain within one order of
magnitude of the orthogonal case, namely $\mathcal A_i\le10$.

\begin{sncorollary}[Theory-calibrated threshold]
\label{cor:app_threshold}
Under the preceding assumptions, the criterion $\mathcal A_i\le10$
is equivalent to
\begin{equation}
    \nu_i\ge\tau_{\mathrm{id}},
    \qquad
    \boxed{\tau_{\mathrm{id}}=\frac1{\sqrt{10}}.}
    \label{eq:app_threshold}
\end{equation}
It is also equivalent to retaining at least $10\%$ of the term's
signature energy outside the span of the remaining terms.
\end{sncorollary}
\begin{proof}
Since $\nu_i>0$, equation~\eqref{eq:app_inflation} gives
\begin{equation}
    \mathcal A_i\le10
    \quad\Longleftrightarrow\quad
    \frac1{\nu_i^2}\le10
    \quad\Longleftrightarrow\quad
    \nu_i^2\ge\frac1{10}
    \quad\Longleftrightarrow\quad
    \nu_i\ge\frac1{\sqrt{10}}.
    \label{eq:app_threshold_equivalence}
\end{equation}
By Proposition~\ref{prop:app_projection}, $\nu_i^2$ is precisely the
fraction of signature energy outside the remaining term span.
This proves both assertions.
\end{proof}

At the threshold boundary, the residual contains exactly $10\%$ of
the term's full signature energy, whereas its norm is
$1/\sqrt{10}$ of the full signature norm.
The worst-case attribution sensitivity is $\sqrt{10}$ times the
orthogonal value, and its square is one order of magnitude larger.

The main text uses the strict qualification rule
$\nu_i>\tau_{\mathrm{id}}$, equivalently $\mathcal A_i<10$ and
$\nu_i^2>0.1$. This excludes equality at the boundary and implies the
non-strict stability requirement above.
The threshold follows from the stated one-order-of-magnitude criterion,
not from algorithm- or dataset-specific threshold optimization.

\subsection{Optimal Deletion and Coefficient Refitting}
\label{app:deletion_refit}

Let $\mathcal S_{-i}=\operatorname{span}\{\phi_j:j\ne i\}$,
as in the main text.
The following refitting problem approximates the original candidate
$s$ in the fixed empirical Sobolev geometry. It allows arbitrary real
outer coefficients on the remaining fixed terms; it is not a refit
of their internal parameters or a minimization of target-label error.

\begin{sntheorem}[Exact deletion-and-refit cost]
\label{thm:app_deletion}
Without requiring $U$ to have full column rank,
\begin{equation}
    \min_{\widetilde s\in\mathcal S_{-i}}
    \|s-\widetilde s\|_{S,X}
    =|b_i|\,\ell_i\nu_i
    =|b_i|\,\|\phi_i\|_{S,X}\nu_i.
    \label{eq:app_deletion}
\end{equation}
The minimum is attained. If $b_i\ne0$, then
\begin{equation}
    \frac{\displaystyle\min_{\widetilde s\in\mathcal S_{-i}}
          \|s-\widetilde s\|_{S,X}^2}
         {b_i^2\|\phi_i\|_{S,X}^2}
    =\nu_i^2.
    \label{eq:app_relative_deletion}
\end{equation}
\end{sntheorem}
\begin{proof}
Write $s=s_{-i}+b_i\phi_i$, where $s_{-i}\in\mathcal S_{-i}$.
By linearity, $\Psi_X(\mathcal S_{-i})=W_{-i}$ and
$\Psi_X(s_{-i})\in W_{-i}$.
Consequently,
\begin{align*}
    \min_{\widetilde s\in\mathcal S_{-i}}
    \|s-\widetilde s\|_{S,X}
    &=\min_{w\in W_{-i}}\|\Psi_X(s)-w\|_2\\
    &=\|(I-P_{-i})\Psi_X(s)\|_2\\
    &=|b_i|\,\|(I-P_{-i})\psi_i\|_2
     =|b_i|\,\ell_i\nu_i.
\end{align*}
The projection $P_{-i}\Psi_X(s)$ belongs to
$\Psi_X(\mathcal S_{-i})$, so at least one $\widetilde s$
attains this distance.
For $b_i\ne0$, squaring and dividing by $b_i^2\ell_i^2>0$ gives
\eqref{eq:app_relative_deletion}.
\end{proof}

Thus novelty measures relative reconstructability, while the absolute
Sobolev deletion cost also depends on $|b_i|\|\phi_i\|_{S,X}$.

\section{Related Work}
\label{app:related_work}

\paragraph{Heuristic search-based SR methods.}
Genetic programming~\citep{burlacu2020operon,cranmer2023interpretable} and Monte Carlo tree search~\citep{sun2022symbolic} are widely used heuristic approaches to SR.
More recent methods employ pretrained neural networks~\citep{shojaee2023transformer,kamienny2023deep}, deep reinforcement learning policies trained online~\citep{petersen2019deep,xu2023rsrm}, and large language models~\citep{shojaee2025llm} to guide the generation and exploration of candidate expressions.
Other methods use neural networks to identify symmetries and separable structures in the data, thereby narrowing the search space~\citep{udrescu2020ai,udrescu2020ai_2}.
However, advances in search mechanisms do not by themselves ensure that the terms within a candidate exhibit distinguishable behavior.
We use Sobolev Novelty to provide complementary term-level structural feedback, guiding the search toward terms that contribute new behavior.

\paragraph{Generative SR methods.}
Generative SR methods pretrain Transformers on large-scale synthetic formula--data pairs to learn the correspondence between numerical observations and symbolic expressions~\citep{biggio2021neural,kamienny2022end}.
These models can generate formulas directly or be integrated with symbolic search to guide candidate generation and search planning~\citep{shojaee2023transformer,kamienny2023deep}.
However, random formula generation does not ensure that the constituent terms exhibit independent functional behavior.
We therefore use Sobolev Novelty to filter out pretraining formulas containing low-novelty terms, reducing term-level behavioral redundancy in the training data.

\paragraph{Evaluation metrics and constraints in SR.}
Benchmarks such as SRBench primarily evaluate SR solutions in terms of predictive accuracy and expression complexity~\citep{la2021contemporary}.
Metrics such as $R^2$ and RMSE assess overall predictive performance, while expression length and depth characterize overall size, but these measures do not directly capture behavioral relationships among terms.
Existing work also introduces structural evaluation and term-level feedback.
EIC assesses computational structural stability through the propagation and amplification of rounding errors~\citep{yu2025beyond}.
IGSR measures predictive influence using the change in validation MSE, $\Delta\mathrm{MSE}$, caused by removing an individual term~\citep{saveliev2026influence}.
PiT-PO~\citep{wang2026llm} identifies potentially redundant terms from the relative magnitudes of fitted coefficients and applies token-level penalties.
However, these criteria do not directly quantify the independent structural information each term contributes relative to the remaining terms.
Sobolev Novelty addresses this gap by quantifying how much of each term's behavior cannot be jointly reconstructed by the remaining terms in an empirical Sobolev space, complementing expression-level evaluation.

\section{Datasets and Evaluation Protocols}
\label{app:datasets}

\subsection{White-Box Tasks}
\label{app:white_box_tasks}

Our SRBench white-box evaluation set comprises 133 scientific equation tasks, including 119 Feynman tasks and 14 Strogatz tasks, each with an analytic ground-truth expression. This set is used for the structural analysis in Section~4.1, the evaluation of GP, MCTS, and their Sobolev-guided variants in Section~4.2, and the testing of pretrained models in Section~4.3.

\subsection{Black-Box Tasks}
\label{app:black_box_tasks}

The black-box evaluation set comprises 122 regression tasks for which no reference analytic expression is provided. This set is used to evaluate GP, MCTS, IGSR, and their Sobolev-guided variants. IGSR and SN-IGSR are evaluated only on the black-box tasks.

\subsection{Pretraining Corpora}
\label{app:pretraining_corpora}

Section~4.3 compares three formula corpora for E2ESR pretraining. Random uses randomly generated formulas. SN-filtered applies Sobolev Novelty filtering to randomly generated formulas: single-term formulas are accepted directly, whereas multi-term formulas are retained only when $\min_i \nu_i > \tau_{\mathrm{id}}$. PhyE2E formulas are taken from the pretraining corpus introduced by \citet{ying2025neural}. All three settings use the same E2ESR architecture, tokenizer, training objective, optimizer, and inference pipeline. We conduct three independent runs per setting and evaluate checkpoints from 1M to 35M accepted samples using mean $R^2$ across the 133 white-box tasks.

\subsection{Periodic-Hill Flow}
\label{app:periodic_hill_flow}

Section~4.4 considers turbulence modeling for periodic-hill flow. The constitutive relations discovered by MCTS, DSRRANS, and their Sobolev-guided variants are embedded in a RANS solver for a posteriori evaluation. Conventional RANS serves as the baseline, and DNS provides the reference. Evaluation considers the reattachment position $x/H$, the maximum recirculation-zone velocity magnitude (in $\mathrm{m/s}$), and the streamwise velocity field. For the first two quantities, prediction errors are computed as absolute deviations from DNS.

\subsection{Evaluation Metrics and Reporting Conventions}
\label{app:evaluation_metrics}

\paragraph{Predictive accuracy and expression complexity.}
Predictive accuracy is measured by the coefficient of determination $R^2$ on the test set. Expression complexity $C$ is defined as the total number of nodes in the parsed expression tree, including operator, variable, and constant nodes. In Figure~2, accuracy ranks are computed from mean $R^2$, and complexity ranks from mean expression complexity. Higher mean $R^2$ and lower mean complexity correspond to better accuracy and complexity ranks, respectively.

\paragraph{Structural evaluation.}
We report the mean term Sobolev Novelty and the proportion of terms satisfying $\nu_i > \tau_{\mathrm{id}}$, where $\tau_{\mathrm{id}} = 1/\sqrt{10}$. The structural analysis in Section~4.1 retains only discovered expressions with test $R^2 \geq 0.8$. Within each task, Sobolev Novelty is computed for the ground-truth equation and all retained expressions using the same input samples. Ground-truth expressions serve solely as post-hoc structural references and are not used to guide or supervise the methods.

\section{Method Integration and Implementation Details}
\label{app:implementation}

\subsection{Integration with Genetic Programming}
\label{app:gp_integration}

\paragraph{Original workflow.}
GP represents candidate expressions as trees and performs population-based search over successive generations. In each generation, the coefficients of candidate expressions are fitted, and parents are selected according to an accuracy--complexity score. Crossover and mutation produce new expressions, which are fitted and evaluated before the population is updated. Elitism retains high-scoring candidates for subsequent generations.

\paragraph{Integration of Sobolev Novelty.}
We augment this workflow with structural screening and reuse. After fitting and evaluating the candidates in each generation, we extract symbolic terms from distinct high-scoring expressions, remove duplicates, and store the terms in a structural archive maintained across generations. Within the archive capacity limit, we prioritize terms according to their Sobolev projection residuals relative to the span of the retained term representations. This favors structures with complementary function-value and gradient behavior.

Starting from the current partial expression, we construct additional candidates using the retained terms. For each new term, we assess its ability to explain the fitting residual and its structural complementarity to existing terms. These assessments yield fit-guided and Sobolev-guided candidate sets, which we merge and deduplicate. Candidate construction also includes bounded structural proposals involving products, rational expressions, trigonometric combinations, and radicals.

We fit and score all additional candidates using the original numerical fitting routine. The final expression is selected from the stored complete candidates according to the original GP accuracy--complexity score, without an additional novelty penalty. Sobolev Novelty enhances GP search by guiding the retention and recombination of candidate structures.

\subsection{Integration with Monte Carlo Tree Search}
\label{app:mcts_integration}

\paragraph{Original workflow.}
MCTS searches for expressions through node selection, expansion, simulation, and reward backpropagation. Node selection determines which branch to explore, expansion generates successor structures, and simulation produces complete candidates for numerical evaluation. After coefficient fitting, candidates are evaluated using an accuracy--complexity score. The resulting rewards are backpropagated through the search tree to inform subsequent node selection.

\paragraph{Integration of Sobolev Novelty.}
We introduce structural screening into candidate evaluation in each round of expansion and simulation. Complete candidates obtained from the current expansion state and subsequent simulations are collected into a common pool. We first retain at most 64 candidates according to the original accuracy--complexity score $R_{\mathrm{base}}$, then rerank them using the Sobolev Novelty of their constituent terms. The structurally guided score is defined as
\begin{equation}
\label{eq:app_mcts_sn_score}
R_{\mathrm{SN}}(s)
=
R_{\mathrm{base}}(s)
-
\frac{0.01}{m}
\sum_{i=1}^{m}
\left[
\max\left(0,1-\frac{\nu_i}{\tau_{\mathrm{id}}}\right)
\right]^2,
\end{equation}
where $m$ is the number of terms in the candidate. The original score identifies candidates with favorable accuracy--complexity trade-offs, and the additional penalty lowers the priority of candidates containing low-novelty terms.

For the top-ranked candidate, we attempt at most one term deletion followed by coefficient refitting. Among the terms satisfying $\nu_i < \tau_{\mathrm{id}}$, we remove the term with the smallest independent contribution $|b_i|\nu_i\|\Psi_X(\phi_i)\|_2$ and refit the remaining coefficients. We accept the modification only if the original accuracy--complexity score does not decrease; otherwise, we retain the original expression.

Finally, we rerank the candidate pool and backpropagate the score $R_{\mathrm{SN}}$ of the selected candidate through the search tree. The original node-selection rule remains unchanged, and Sobolev Novelty influences subsequent exploration through candidate evaluation and reward backpropagation.

\subsection{Integration with IGSR}
\label{app:igsr_integration}

\paragraph{Original workflow.}
IGSR combines LLM-based symbolic term generation with MCTS. Tree search selects a candidate for expansion, and the LLM proposes new symbolic terms that are combined with existing terms to form an additive model. After numerical coefficient fitting, the original term-wise influence mechanism screens terms based on the change in validation error caused by removing each term. The resulting candidate is refitted, and its validation score guides subsequent node selection and expansion.

\paragraph{Integration of Sobolev Novelty.}
We perform structural diagnosis before the LLM proposes new terms. We compute the term-wise Sobolev Novelty of the parent candidate on a fixed set of training samples used during search. Terms satisfying $\nu_i \leq \tau_{\mathrm{id}}$ are flagged as low-novelty terms. The resulting feedback lists at most three low-novelty terms and two high-novelty reference terms and is appended to the prompt for the next round of term generation.

The feedback identifies terms whose function values and local variation can be jointly approximated by the remaining terms and guides the LLM to propose structures with complementary variation patterns. When the parent candidate is empty, diagnosis fails, or no low-novelty terms are identified, we retain the original prompt without adding this feedback.

After term generation, the original IGSR workflow continues with influence-based screening, coefficient fitting, validation evaluation, and tree-search updates. To keep modifications to IGSR minimal, we introduce Sobolev Novelty only through term-generation feedback. The validation objective and tree-search scoring remain unchanged, and no additional term deletion or sibling-candidate reranking is performed. Both settings use the same LLM and logical token budget, with the additional feedback text included in this budget.

\subsection{Integration with E2ESR}
\label{app:e2esr_integration}

\paragraph{Original workflow.}
E2ESR is pretrained on large-scale synthetic formula--data pairs to learn a mapping from numerical observations to symbolic expressions. Pretraining comprises random formula generation, numerical sample construction, sequence encoding, and model updates. The generator constructs expressions online using variables, constants, and operators. For each expression, it samples input points and computes the corresponding outputs, pairing these observations with the symbolic sequence to form a training example. At inference time, the model generates candidate expressions from the given observations. These candidates undergo numerical constant optimization and scoring to select the final output.

\paragraph{Integration of Sobolev Novelty.}
We introduce formula-level rejection sampling between numerical sample generation and model training. For each randomly generated expression, we extract the top-level additive terms by distributing multiplication over addition. Powers and expressions within functions are left unexpanded, and no trigonometric identities are applied.

We evaluate the function values and first-order input gradients of each term at the points sampled for the corresponding formula, using all available points up to a maximum of 200. Gradients are computed analytically using the chain rule on the expression tree and converted to derivatives with respect to standardized input coordinates. Structural evaluation is restricted to a common subset of points at which all terms have finite function values and gradients. All term representations share a candidate-level output scale and use weights $\lambda_0 = \lambda_1 = 1$. Formulas are rejected if fewer than 32 valid samples remain or valid representations cannot be constructed.

Following the definition in Section~3, we compute the leave-one-term-out Sobolev Novelty of each term relative to the remaining terms. Among formulas that pass the numerical validity checks, single-term formulas are accepted with $\nu_1 = 1$, and multi-term formulas are accepted only if $\min_i \nu_i > \tau_{\mathrm{id}}$, where $\tau_{\mathrm{id}} = 1/\sqrt{10}$. If any term fails to exceed the threshold, the entire formula is rejected without term deletion or refitting, and the generator proceeds to the next candidate.

Accepted samples enter the original pipeline for encoding, batch construction, and model updates. The E2ESR architecture, tokenizer, training objective, optimizer, and inference-time procedures for constant optimization and candidate ranking remain unchanged. No additional Sobolev loss or inference-time reranking is introduced. The training budget counts only accepted samples used for model updates. Sobolev Novelty improves the final performance of E2ESR through pretraining data selection alone.

\subsection{Turbulence Closure Modeling for Periodic-Hill Flow}
\label{app:turbulence_details}

\subsubsection{Background of the Turbulence Modeling Task}
\label{app:turbulence_background}

\paragraph{Physical background.}
The Reynolds-averaged Navier--Stokes (RANS) approach predicts mean flow quantities, such as velocity and pressure, by solving averaged flow equations. Reynolds decomposition of the instantaneous velocity and averaging of the momentum equations introduce the Reynolds stress tensor
\begin{equation}
\label{eq:app_reynolds_stress}
\tau_{ij} = \overline{u_i'u_j'},
\end{equation}
where $u_i'$ denotes the velocity fluctuation in the $i$th direction and the overbar denotes statistical averaging. Reynolds stresses represent the contribution of turbulent fluctuations to mean momentum transport. An additional constitutive relation is required to determine these stresses and close the mean flow equations. Direct numerical simulation (DNS) provides high-fidelity reference statistics for learning and evaluating closure relations.

Conventional linear eddy-viscosity models express the anisotropic part of the Reynolds stress as proportional to the mean strain rate:
\begin{equation}
\label{eq:app_linear_eddy_viscosity}
\boldsymbol{\tau}
=
\frac{2}{3}k\mathbf I
-
\nu_t
\left[
\nabla\overline{\mathbf u}
+
\left(\nabla\overline{\mathbf u}\right)^{\!\top}
\right],
\end{equation}
where $k$ is the turbulent kinetic energy, $\nu_t$ is the eddy viscosity, and $\mathbf I$ is the identity tensor. This restricted stress--strain relationship can lead to substantial errors in flows with pronounced separation, streamline curvature, and stress anisotropy. DSRRANS retains the linear eddy-viscosity contribution and adds a nonlinear stress correction represented by explicit symbolic expressions.

\paragraph{Periodic-hill flow and case configuration.}
Periodic-hill flow develops in a channel with a flat upper wall and a periodically undulating lower wall. Downstream of a hill crest, an adverse pressure gradient decelerates the near-wall flow, producing a separated shear layer and a recirculation region. The flow subsequently reattaches downstream. The extent of the recirculation region, the reattachment position, and the velocity distribution reflect how well a closure model represents stress and momentum transport in separated flows. DSRRANS also uses this flow for turbulence modeling.

For tensor regression, we use the publicly available DSRRANS data with geometry parameter $\alpha = 0.8$ and Reynolds number
\begin{equation}
\label{eq:app_hill_reynolds_number}
Re_H = \frac{U_bH}{\nu} = 5600,
\end{equation}
where $H$ is the hill height, $U_b$ is the volume-averaged velocity, and $\nu$ is the kinematic viscosity. The parameter $\alpha$ scales the hills horizontally, controlling their steepness and the streamwise period length. Periodic boundary conditions are imposed in the streamwise direction, and no-slip conditions are applied at both walls. The configuration has a statistically two-dimensional mean flow, while the Reynolds stress is represented by a three-dimensional tensor that includes the spanwise normal-stress component.

\paragraph{Local flow representation.}
In the publicly available training data, the input features and tensor bases are derived from a standard $k$--$\varepsilon$ RANS solution, where $\varepsilon$ is the turbulent kinetic energy dissipation rate. Following the notation of DSRRANS, we use the timescale $k/\varepsilon$ to define the dimensionless strain-rate and rotation-rate tensors:
\begin{equation}
\label{eq:app_strain_rotation_tensors}
\begin{aligned}
\mathbf S
&=
\frac{k}{2\varepsilon}
\left[
\nabla\overline{\mathbf u}
+
\left(\nabla\overline{\mathbf u}\right)^{\!\top}
\right],
\\
\mathbf R
&=
\frac{k}{2\varepsilon}
\left[
\nabla\overline{\mathbf u}
-
\left(\nabla\overline{\mathbf u}\right)^{\!\top}
\right].
\end{aligned}
\end{equation}
The corresponding scalar invariants are
\begin{equation}
\label{eq:app_flow_invariants}
I_1 = \operatorname{tr}(\mathbf S^2),
\qquad
I_2 = \operatorname{tr}(\mathbf R^2).
\end{equation}
Before symbolic regression, we transform these invariants as
\begin{equation}
\label{eq:app_invariant_transform}
x_1 = \tanh(I_1/2),
\qquad
x_2 = \tanh(I_2/2),
\end{equation}
to compress their numerical ranges. The transformed flow invariants $x_1$ and $x_2$ serve as the regression inputs, and the tensor-basis values are taken directly from the dataset.

\paragraph{Regression target and data composition.}
We follow the dimensionless stress definition used in the DSRRANS data:
\begin{equation}
\label{eq:app_stress_decomposition}
\begin{aligned}
\mathbf b
&=
\frac{\boldsymbol{\tau}}{k}
-
\frac{2}{3}\mathbf I,
\\
\mathbf b^{\parallel}
&=
-2C_\mu\mathbf S,
\\
\mathbf b^{\perp}
&=
\mathbf b-\mathbf b^{\parallel},
\end{aligned}
\end{equation}
where $C_\mu = 0.09$. The regression target $\mathbf b^{\perp}$ is the stress correction relative to the linear eddy-viscosity contribution.

The DSRRANS data preparation pipeline interpolates DNS results onto the RANS mesh to construct the nonlinear stress-correction targets. We use the publicly available arrays of invariants, tensor bases, and correction targets, comprising 9,600 local samples. Each sample contains two regression invariants, three $3\times3$ tensor bases, and one $3\times3$ target tensor. The three coefficient functions are fitted jointly by minimizing the prediction error of the assembled tensor over its $11$, $12$, $22$, and $33$ components.

\paragraph{Closure modeling and flow-field evaluation.}
Stress regression evaluates the fit of the symbolic model to the stress corrections at the given samples. A posteriori flow-field evaluation assesses mean-flow predictions after the learned closure is incorporated into the RANS equations. Section~4.4 reports a posteriori comparisons with DNS for the reattachment position $x/H$, the maximum recirculation-zone velocity magnitude, and the streamwise velocity field. The reattachment position characterizes the downstream extent of the separation region, and the velocity magnitude characterizes recirculation strength. For both scalar quantities, errors are measured as absolute deviations from the DNS reference.

\subsubsection{Closure Search Using the DSRRANS Tensor Representation}
\label{app:dsrrans_integration}

\paragraph{Original method and tensor representation.}
DSRRANS constructs turbulence closures by learning symbolic coefficient functions in a tensor-basis expansion. The original search procedure uses an LSTM to generate candidate expressions for these functions and optimizes their numerical constants. Rewards computed from tensor-prediction errors then guide updates to the generation policy through risk-seeking policy gradients. For the statistically two-dimensional flow considered here, the three tensor bases are
\begin{equation}
\label{eq:app_tensor_basis}
\begin{aligned}
\mathbf T_1 &= \mathbf S,
\\
\mathbf T_2 &= \mathbf S\mathbf R-\mathbf R\mathbf S,
\\
\mathbf T_3
&=
\mathbf S^2
-
\frac{1}{3}\mathbf I\operatorname{tr}(\mathbf S^2).
\end{aligned}
\end{equation}
The nonlinear closure term is expressed as
\begin{equation}
\label{eq:app_tensor_closure}
\widehat{\mathbf b}^{\perp}
=
\sum_{k=1}^{3}G_k(x_1,x_2)\mathbf T_k,
\end{equation}
where $G_k$ denotes a scalar coefficient function to be identified and $x_1,x_2$ are the transformed invariants defined above.

\paragraph{Forward-selection workflow.}
The offline search implementation described here uses this tensor representation with forward selection over a bounded monomial dictionary. The dictionary contains bivariate monomials of total degree at most 12. Each candidate term is multiplied by the tensor basis associated with its coefficient channel to form a tensor contribution. The baseline without Sobolev guidance evaluates each candidate by the reduction in training squared error obtained after adding that candidate and jointly refitting all coefficients. It selects the term with the largest reduction and then jointly refits the coefficients of all selected terms.

\paragraph{Integration of Sobolev Novelty.}
We incorporate Sobolev Novelty into next-term selection after computing candidate gains with the baseline procedure. Among candidates whose error reduction is at least $94\%$ of the largest reduction at the current step, we select the term with the highest conditional Sobolev Novelty. We then jointly refit the coefficients of all selected terms. This rule uses Sobolev Novelty to guide selection among candidates with comparable fitting gains.

Structural evaluation uses the tensor contribution $\phi(x_1,x_2)\mathbf T_k$ of each candidate. The reference space is spanned only by the representations of previously selected terms in the same coefficient channel. The value block records the tensor contribution at the sampled inputs. To construct the gradient block, we differentiate the coefficient function with respect to the input invariants and multiply these derivatives by the tensor basis, which remains fixed during differentiation. Conditional novelty is the normalized projection residual of the candidate representation relative to this reference space and serves as a continuous ranking score.

When further model compression is required, both settings use the same backward term-deletion rule, guided by training error alone, and jointly refit the coefficients after each deletion.

\section{Detailed Experimental Results}
\label{app:detailed_results}

\subsection{Per-Method Term Qualification Statistics}
\label{app:qualification_statistics}

Table~\ref{tab:term_qualification_by_method} reports the complete
per-method term qualification statistics used in Section~\ref{exp1}.
We consider only multi-term expressions, since a single-term expression
has no remaining-term reference space and is assigned a novelty of one
by convention. For discovered expressions, we retain candidates with
test $R^2 \ge 0.8$, following the protocol in the main text. A term is
qualified when
$\nu_i > \tau_{\mathrm{id}} = 1/\sqrt{10}$.

Under this criterion, $87$ of the $94$ ground-truth terms are qualified,
corresponding to a qualification rate of $92.6\%$. Across the 15 SR
methods, the macro-average qualification rate is $38.3\%$, yielding a
difference of $54.3$ percentage points from the ground truth.
Per-method counts and qualification rates are reported in
Table~\ref{tab:term_qualification_by_method}.

\subsection{Sobolev Novelty as Search Guidance}
\label{app:search_results}

Corresponding to Figure~\ref{fig:search_guidance} in the main text,
Tables~\ref{tab:app_whitebox_results} and~\ref{tab:app_blackbox_results}
provide the complete experimental results on the white-box and black-box
tasks, respectively. The tables report the mean $R^2$, mean expression
complexity, and mean Sobolev Novelty of each method, together with the
relative changes between GP, MCTS, and IGSR and their Sobolev-guided
variants.

\begin{table}[t]
    \centering
    \caption{
        Per-method term qualification statistics at the
        theory-calibrated threshold
        $\tau_{\mathrm{id}}=1/\sqrt{10}$.
        A term is qualified when
        $\nu_i>\tau_{\mathrm{id}}$.
    }
    \label{tab:term_qualification_by_method}
    \small
    \setlength{\tabcolsep}{7pt}
    \renewcommand{\arraystretch}{1.06}
    \begin{tabular}{@{}lrrrr@{}}
        \toprule
        Method
        & \shortstack{Multi-term\\Expressions}
        & \shortstack{Total\\Terms}
        & \shortstack{Qualified\\Terms}
        & \shortstack{Qualification\\Rate (\%)} \\
        \midrule

        \textbf{Ground Truth}
        & 41
        & 94
        & 87
        & \textbf{92.6} \\

        \midrule
        AFP
        & 451
        & 1{,}496
        & 1{,}039
        & 69.5 \\

        AFP-FE
        & 563
        & 1{,}814
        & 1{,}295
        & 71.4 \\

        AIFeynman2
        & 388
        & 6{,}130
        & 448
        & 7.3 \\

        BSR
        & 524
        & 4{,}335
        & 1{,}657
        & 38.2 \\

        DSR
        & 339
        & 990
        & 677
        & 68.4 \\

        E2ESR
        & 672
        & 5{,}508
        & 646
        & 11.7 \\

        EPLEX
        & 434
        & 1{,}810
        & 969
        & 53.5 \\

        FEAT
        & 838
        & 12{,}516
        & 1{,}534
        & 12.3 \\

        GP-GOMEA
        & 737
        & 4{,}815
        & 1{,}434
        & 29.8 \\

        GPlearn
        & 459
        & 2{,}871
        & 1{,}253
        & 43.6 \\

        NeurSR
        & 318
        & 1{,}325
        & 428
        & 32.3 \\

        Operon
        & 964
        & 12{,}223
        & 1{,}280
        & 10.5 \\

        PySR
        & 690
        & 2{,}635
        & 1{,}475
        & 56.0 \\

        RSRM
        & 426
        & 1{,}374
        & 844
        & 61.4 \\

        SBP-GP
        & 560
        & 19{,}072
        & 1{,}524
        & 8.0 \\

        \midrule
        \textbf{Mean over 15 SR methods}
        & --
        & --
        & --
        & \textbf{38.3} \\

        \bottomrule
    \end{tabular}

    \vspace{2pt}
    \parbox{0.94\linewidth}{
        \footnotesize
        \textit{Note.}
        The mean over SR methods is the unweighted macro-average of
        the 15 method-level qualification rates, rather than a pooled
        average over all terms.
    }
\end{table}

\begin{table}[t]
    \centering
    \small
    \setlength{\tabcolsep}{5pt}
    \renewcommand{\arraystretch}{1.06}

    \begin{threeparttable}
    \caption{White-box results under noise-free conditions.}
    \label{tab:app_whitebox_results}

    \begin{tabular*}{\linewidth}{
        @{\extracolsep{\fill}}c l r r r@{}
    }
        \toprule
        \textbf{Type}
        & \textbf{Algorithm}
        & \textbf{Mean $R^2$} $\uparrow$
        & \textbf{Complexity} $\downarrow$
        & \textbf{SN} $\uparrow$ \\
        \midrule

        Regression
        & \textbf{FEAT}
        & 0.9192 & 195.300 & 0.1686 \\
        \midrule

        \multirow{2}{*}{Generative}
        & \textbf{E2ESR}
        & 0.8194 & 89.630 & 0.3675 \\
        & \textbf{NeurSR}
        & 0.4090 & 31.420 & 0.4348 \\
        \midrule

        \multirow{18}{*}{Search}
        & \textbf{AFP}
        & 0.9553 & 37.010 & 0.8185 \\
        & \textbf{AFP-FE}
        & 0.9767 & 40.630 & 0.8197 \\
        & \textbf{AIFeynman2}
        & 0.9000 & 113.200 & 0.7781 \\
        & \textbf{BSR}
        & 0.6807 & 26.950 & 0.3684 \\
        & \textbf{DSR}
        & 0.8351 & 14.940 & 0.8043 \\
        & \textbf{EPLEX}
        & 0.9682 & 52.640 & 0.7882 \\
        & \textbf{GP-GOMEA}
        & 0.9953 & 34.770 & 0.5126 \\
        & \textbf{GPlearn}
        & 0.8688 & 67.710 & 0.7285 \\
        & \textbf{Operon}
        & 0.9888 & 68.730 & 0.3700 \\
        & \textbf{PySR}
        & 0.9587 & 9.219 & 0.3535 \\
        & \textbf{RSRM}
        & 0.7750 & 13.400 & 0.7286 \\
        & \textbf{SBP-GP}
        & 0.9930 & 513.400 & 0.2124 \\

        \cmidrule(lr){2-5}

        & \textbf{GP}
        & 0.9652 & 18.421 & 0.6647 \\
        & \textbf{SN-GP}
        & 0.9960 & 16.485 & 0.7790 \\
        & $\Delta$ (\%)
        & $+3.19\%$ & $-10.51\%$ & $+17.20\%$ \\

        \cmidrule(lr){2-5}

        & \textbf{MCTS}
        & 0.9229 & 18.109 & 0.7790 \\
        & \textbf{SN-MCTS}
        & 0.9356 & 17.549 & 0.7845 \\
        & $\Delta$ (\%)
        & $+1.38\%$ & $-3.09\%$ & $+0.71\%$ \\

        \bottomrule
    \end{tabular*}

    \begin{tablenotes}[flushleft]
        \footnotesize
        \item[] SN denotes Sobolev Novelty.
        $\Delta = 100(x_{\mathrm{SN}}-x_{\mathrm{base}})
        /x_{\mathrm{base}}$
        denotes the relative change computed from the displayed values.
    \end{tablenotes}
    \end{threeparttable}
\end{table}

\begin{table}[t]
    \centering
    \small
    \setlength{\tabcolsep}{5pt}
    \renewcommand{\arraystretch}{1.06}

    \begin{threeparttable}
    \caption{Black-box results.}
    \label{tab:app_blackbox_results}

    \begin{tabular*}{\linewidth}{
        @{\extracolsep{\fill}}c l r r r@{}
    }
        \toprule
        \textbf{Type}
        & \textbf{Algorithm}
        & \textbf{Mean $R^2$} $\uparrow$
        & \textbf{Complexity} $\downarrow$
        & \textbf{SN} $\uparrow$ \\
        \midrule

        Regression
        & \textbf{FEAT}
        & 0.7621 & 82.489 & 0.6555 \\
        \midrule

        \multirow{2}{*}{Generative}
        & \textbf{E2ESR}
        & 0.3612 & 61.094 & 0.6331 \\
        & \textbf{NeurSR}
        & 0.1228 & 13.333 & 0.6590 \\
        \midrule

        \multirow{22}{*}{Search}
        & \textbf{AFP}
        & 0.6333 & 34.894 & 0.7094 \\
        & \textbf{AFP-FE}
        & 0.6400 & 36.039 & 0.7078 \\
        & \textbf{AIFeynman2}
        & 0.2110 & 2239.752 & 0.4286 \\
        & \textbf{BSR}
        & 0.2725 & 22.521 & 0.6740 \\
        & \textbf{DSR}
        & 0.5625 & 9.465 & 0.7467 \\
        & \textbf{EPLEX}
        & 0.7372 & 53.142 & 0.6592 \\
        & \textbf{FFX}
        & 0.5575 & 1561.893 & 0.1549 \\
        & \textbf{GP-GOMEA}
        & 0.7381 & 30.266 & 0.5184 \\
        & \textbf{GPlearn}
        & 0.5390 & 19.058 & 0.7591 \\
        & \textbf{ITEA}
        & 0.6295 & 116.692 & 0.7147 \\
        & \textbf{MRGP}
        & 0.5300 & 10802.169 & N/A \\
        & \textbf{Operon}
        & 0.7945 & 65.689 & 0.3210 \\
        & \textbf{SBP-GP}
        & 0.7869 & 634.019 & 0.2915 \\

        \cmidrule(lr){2-5}

        & \textbf{GP}
        & 0.7194 & 33.733 & 0.5269 \\
        & \textbf{SN-GP}
        & 0.7298 & 35.440 & 0.5907 \\
        & $\Delta$ (\%)
        & $+1.45\%$ & $+5.06\%$ & $+12.11\%$ \\

        \cmidrule(lr){2-5}

        & \textbf{MCTS}
        & 0.5457 & 38.301 & 0.5677 \\
        & \textbf{SN-MCTS}
        & 0.5833 & 37.060 & 0.5796 \\
        & $\Delta$ (\%)
        & $+6.89\%$ & $-3.24\%$ & $+2.10\%$ \\

        \cmidrule(lr){2-5}

        & \textbf{IGSR}
        & 0.5640 & 6.000 & 0.1908 \\
        & \textbf{SN-IGSR}
        & 0.6474 & 5.992 & 0.4218 \\
        & $\Delta$ (\%)
        & $+14.79\%$ & $-0.13\%$ & $+121.07\%$ \\

        \bottomrule
    \end{tabular*}

    \begin{tablenotes}[flushleft]
        \footnotesize
        \item[] SN denotes Sobolev Novelty.
        $\Delta = 100(x_{\mathrm{SN}}-x_{\mathrm{base}})
        /x_{\mathrm{base}}$
        denotes the relative change computed from the displayed values.
        N/A denotes an unavailable value.
    \end{tablenotes}
    \end{threeparttable}
\end{table}

\section{Ablation of First-Order Derivatives}
\label{app:derivative_ablation}

\paragraph{Experimental setup.}
We assess the contribution of first-order derivatives to search accuracy
and evaluation throughput under the same wall-clock search budget.
We compare the complete SN-MCTS with \emph{w/o first derivatives},
which removes the gradient block from the term representation while
retaining the function-value block and term-level projection guidance.
The remaining search procedure follows the SN-MCTS integration in
Appendix~D.2. This comparison evaluates the benefit of derivative
information beyond value-only guidance and its practical computational
impact. We report mean test NMSE and the mean number of SN calls per run,
with the latter measuring how many novelty evaluations are completed
within the fixed time budget.

\paragraph{Results.}
As shown in Figure~\ref{fig:derivative_ablation}, removing first-order
derivatives increases mean NMSE from $0.0644$ to $0.0723$, a relative
increase of $12.3\%$. Under the same time budget, the complete SN-MCTS
performs $5229.4$ SN calls per run on average, compared with $5766.5$
for the ablated variant. The complete method therefore retains $90.7\%$
of the SN call count, corresponding to a $9.3\%$ reduction.
Although derivative evaluation introduces additional computation,
its observed impact on the number of SN evaluations completed within
the fixed budget is modest. The complete method achieves lower
prediction error while preserving most of the evaluation throughput,
supporting the practical value of first-order information for
term-level search guidance.

\begin{figure}[htbp]
    \centering
    \includegraphics[width=0.92\linewidth]{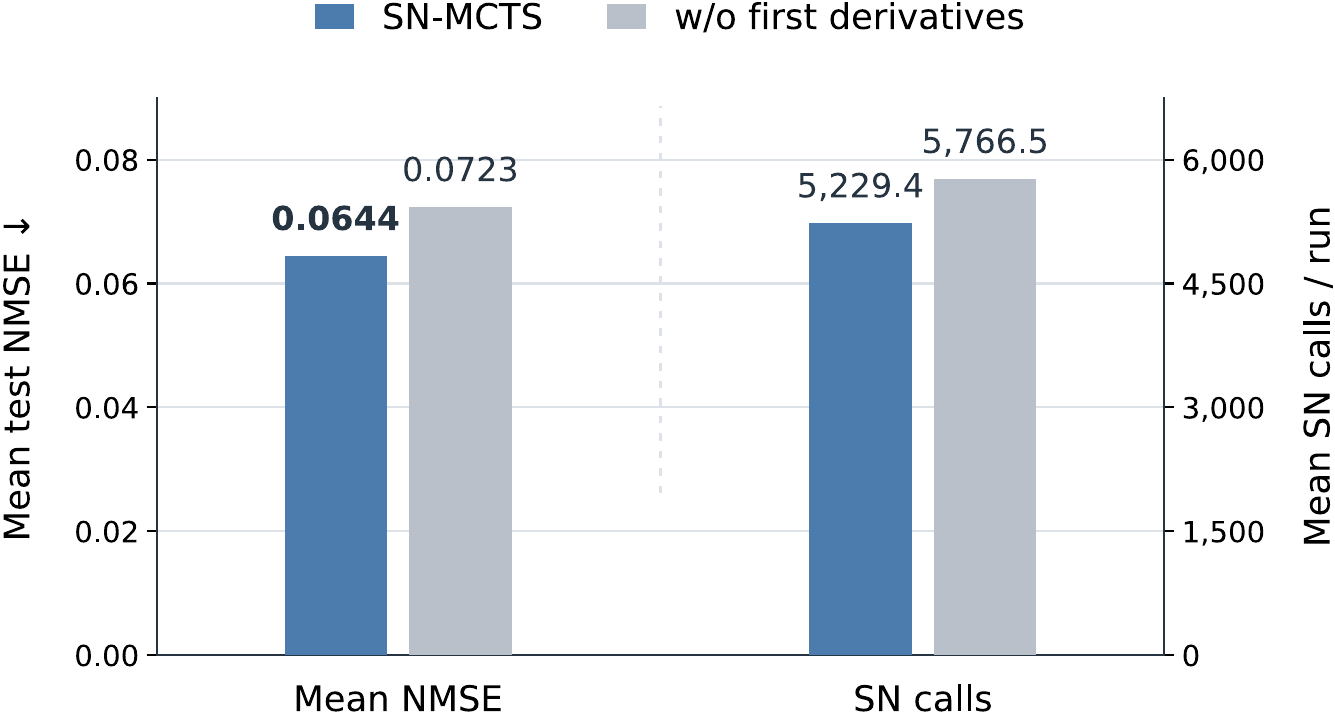}
    \caption{Mean test NMSE (left group, left axis) and mean SN calls
    per run (right group, right axis) under the same time budget.}
    \label{fig:derivative_ablation}
\end{figure}

\end{document}